\documentclass[11pt,letterpaper]{article}

\usepackage[T1]{fontenc}
\usepackage{lmodern}
\usepackage[margin=1in]{geometry}
\usepackage{microtype}
\usepackage{mathtools}
\usepackage{amsfonts}
\usepackage{amssymb}
\usepackage{amsthm}
\usepackage{array}
\usepackage{graphicx}
\usepackage{subcaption}
\usepackage{enumitem}
\usepackage{xcolor}
\usepackage{hyperref}
\usepackage[nameinlink,capitalize,noabbrev]{cleveref}

\hypersetup{
  colorlinks=true,
  linkcolor=blue!55!black,
  citecolor=green!35!black,
  urlcolor=blue!60!black,
  pdftitle={A Smooth Phase-Separation Model for Weak-Boundary Segmentation of Homogeneous Structures},
  pdfauthor={Zihan Li, Jiebao Sun, Fanghui Song, and Zhichang Guo}
}

\setlist[enumerate]{leftmargin=.5in}
\setlist[itemize]{leftmargin=.5in}
\newtheorem{theorem}{Theorem}[section]
\newtheorem{lemma}[theorem]{Lemma}
\newtheorem{proposition}[theorem]{Proposition}
\theoremstyle{definition}
\newtheorem{definition}[theorem]{Definition}
\theoremstyle{remark}
\newtheorem{remark}[theorem]{Remark}

\newcommand{\R}{\mathbb{R}}

\newcommand{\dd}{\,\mathrm{d}}
\newcommand{\Nop}{\mathcal{N}}
\newcommand{\E}{\mathcal{E}}
\newcommand{\Etilde}{\widetilde{\mathcal{E}}}
\DeclarePairedDelimiter{\abs}{\lvert}{\rvert}
\DeclarePairedDelimiter{\norm}{\lVert}{\rVert}
\newcommand{\ip}[2]{\left(#1,#2\right)}
\newcommand{\dual}[2]{\left\langle #1,#2\right\rangle}

\title{A Smooth Phase-Separation Model for Weak-Boundary Segmentation of Homogeneous Structures}
\author{%
  Zihan Li\textsuperscript{1}\quad
  Jiebao Sun\textsuperscript{1}\quad
  Fanghui Song\textsuperscript{1}\quad
  Zhichang Guo\textsuperscript{1,*}\\[0.6em]
  \small\textsuperscript{1}School of Mathematics, Harbin Institute of Technology,
  Harbin 150001, China\\
  \small\texttt{25B912030@stu.hit.edu.cn},
  \texttt{sunjiebao@hit.edu.cn},
  \texttt{fanghuisong@stu.hit.edu.cn}\\
  \small\textsuperscript{*}Corresponding author:
  \texttt{mathgzc@hit.edu.cn}
}
\date{}

\begin{document}

\maketitle

\begin{abstract}
Segmentation of adjacent structures with similar intensity distributions remains a challenging problem in image analysis, particularly when object boundaries are weak or ambiguous. Under such conditions, classical variational models may suffer from degenerated image-driven forces, leading to boundary leakage or undesired merging of neighboring regions. To address these limitations, we propose a smooth phase-separation variational model based on the Cahn--Hilliard equation for weak-boundary segmentation of homogeneous-appearance structures. The proposed framework integrates softmax-based region fitting with Cahn--Hilliard phase-field regularization to maintain interface discrimination under weak image-driven forces. We further introduce a mixed $L^2-H^{-1}$ gradient flow, which preserves higher-order interfacial regularization while allowing adaptive changes of phase masses, establish the continuous energy dissipation law, and prove the existence and uniqueness of weak solutions in the natural solution class. For numerical computation, we develop a stabilized scalar auxiliary variable (SAV) scheme that is linear, FFT-based, and satisfies a modified discrete energy dissipation law. Numerical experiments on synthetic and medical images demonstrate that the proposed method effectively separates adjacent homogeneous structures across weak boundaries and achieves competitive segmentation accuracy and improved boundary localization compared with representative variational, phase-field, and deep learning methods.
\end{abstract}



\bigskip
\section{Introduction}

Image segmentation is a fundamental task in image processing, with broad applications in medical imaging, object recognition, quantitative biomarker extraction, and computer-assisted analysis~\cite{phameSurveyCurrentMethods,uijlingsSelectiveSearchObject2013,pereiraBrainTumorSegmentation2016,christAutomaticLiverTumor2017,doiComputeraidedDiagnosisMedical2007}. In many practical scenarios, one needs to distinguish structures that have homogeneous or highly similar visual appearances but belong to different semantic or anatomical regions. This homogeneous-appearance yet semantically distinct segmentation problem is particularly challenging when adjacent structures have overlapping intensity distributions and are separated only by weak, blurred, or noisy boundaries~\cite{chunmingliMinimizationRegionScalableFitting2008a,huangActiveContourModel2015,wangNewRegionScalableDiscriminant2014,luAccurate3DBone2016,pangIntensityInhomogeneityImage2023}. In such cases, image gradients, regional statistics, and boundary cues often become unreliable, making it difficult to construct a stable image-driven force for interface evolution.

Numerous image segmentation methods have been developed over the past decades, including thresholding~\cite{otsuThresholdSelectionMethod1979}, clustering~\cite{comaniciuMeanShiftRobust2002}, graph-based methods~\cite{shiNormalizedCutsImage}, variational and level-set methods~\cite{casellesGeodesicActiveContours1995,mumfordOptimalApproximationsPiecewise1989,chanActiveContoursEdges2001}, and, more recently, deep learning approaches based on convolutional, transformer, and diffusion architectures~\cite{ronnebergerUNetConvolutionalNetworks2015a,chenTransUNetTransformersMake2021,caoSwinUnetUnetlikePure2021,wuMedSegDiffMedicalImage}. Interactive segmentation methods have also been investigated by incorporating sparse user guidance or geometric constraints into the segmentation process~\cite{zhangQISInteractiveSegmentation2025}. Among these approaches, variational and level-set methods remain particularly attractive because of their solid mathematical foundation, flexible geometric representation, and compatibility with efficient numerical optimization.

Representative variational models include the geodesic active contour (GAC) model~\cite{casellesGeodesicActiveContours1995}, which exploits image gradients or edge indicator functions to drive contour evolution toward object boundaries, and the Mumford--Shah (MS)~\cite{mumfordOptimalApproximationsPiecewise1989} and Chan--Vese (CV)~\cite{chanActiveContoursEdges2001} models, which formulate segmentation through region-fitting energies and therefore reduce the reliance on explicit edge detection. Related binary level-set formulations further simplify the numerical implementation of Mumford--Shah segmentation~\cite{lieBinaryLevelSet2006}. To improve robustness against intensity inhomogeneity, local fitting models, such as the region-scalable fitting (RSF) model~\cite{chunmingliMinimizationRegionScalableFitting2008a} and the local binary fitting (LBF) model~\cite{liImplicitActiveContours2007}, incorporate local image statistics into the fitting energy and achieve improved performance on images with spatially varying intensities.

Beyond region-fitting energies, geometric regularization has become an important component of variational image segmentation. Shape priors, including convexity and star-shape constraints, have been incorporated to improve geometric consistency~\cite{siuImageSegmentationPartial2020,zhangRegistrationBasedStarShapeSegmentation2025}. Topology-preserving and connected-component-preserving methods can reduce certain structural errors by imposing topological constraints~\cite{dengConnectedComponentPreservingImage2025,zhangQISInteractiveSegmentation2025}. Total variation, Euler elastica, and other higher-order regularization models further enhance boundary smoothness and contour quality while admitting efficient numerical solvers based on augmented Lagrangian, dual, or Bregman-type optimization techniques~\cite{wuAugmentedLagrangianMethod2010,taiFastAlgorithmEulers2011}. More recently, molecular beam epitaxy (MBE) regularization has been introduced into level-set segmentation to improve interface regularity, stabilize contour evolution, and alleviate the need for re-initialization~\cite{songReinitializationFreeLevelSet2024,songEffectiveLevelSet2025}.

Although classical variational models have achieved considerable success, their interface evolution is still largely governed by image-dependent driving forces, including edge attraction, region-fitting contrast, and local intensity statistics. Consequently, their performance deteriorates when discriminative image information becomes weak or ambiguous. For edge-based models, weak or discontinuous gradients often lead to boundary leakage or inaccurate contour localization~\cite{kassSnakesActiveContour1988,casellesGeodesicActiveContours1995}. Region-based models instead rely on differences between fitting energies inside and outside the evolving contour. When neighboring regions exhibit similar intensities, overlapping gray-level distributions, or strong intra-region variations, this driving force may become insufficient to separate adjacent structures effectively. This difficulty is especially pronounced for homogeneous-appearance structures, where the target regions may have similar mean intensities and locally overlapping statistics. In this regime, CV-type global fitting forces can become nearly degenerate, and RSF-type local fitting forces may still mix information from neighboring structures when the separating gap is narrow or weak~\cite{liImplicitActiveContours2007,chunmingliMinimizationRegionScalableFitting2008a}. Topology-preserving or connected-component-preserving methods can impose structural constraints, but the resulting interface evolution may become rigid and may not provide sufficiently smooth or accurate boundaries in weak-boundary regions. Therefore, the key difficulty is to design a segmentation model that can both separate adjacent homogeneous structures and maintain smooth, coherent interfaces when the image-driven force is unreliable.


Deep learning methods have significantly advanced medical image segmentation by learning discriminative image representations from annotated datasets~\cite{ronnebergerUNetConvolutionalNetworks2015a,chenTransUNetTransformersMake2021,caoSwinUnetUnetlikePure2021,wuMedSegDiffMedicalImage}. Nevertheless, their performance often depends on large amounts of representative training data and lacks the explicit geometric interpretation provided by variational formulations. Hybrid approaches combining deep learning with variational models have therefore attracted increasing attention~\cite{zhengUnsupervisedDeepLearning2022}. These developments further suggest that variational models equipped with robust geometric regularization remain valuable, particularly when annotated data are limited or reliable boundary evolution is required~\cite{calderVariationalApproachBone2011}.

To address these limitations, we develop a smooth phase-separation variational model based on the Cahn--Hilliard equation for weak-boundary segmentation of homogeneous structures. Inspired by the use of Cahn--Hilliard phase-field models in image processing~\cite{bertozziInpaintingBinaryImages2007,lipkovaAutomatedUnsupervisedSegmentation2017}, the proposed framework combines a smooth softmax-based probabilistic representation with a physically motivated phase-field regularization. The softmax formulation assigns differentiable membership weights to multiple phases and allows them to compete continuously, while a distance-prior map \(D_a\) incorporates spatial preference into the region-fitting cost for target localization. As a result, the data fidelity term remains smooth with respect to the phase variables and balances intensity fitting with distance-guided localization. The Cahn--Hilliard regularization introduces gradient energy and double-well potential energy to promote phase separation and coherent diffuse interfaces~\cite{cahnFreeEnergyNonuniform1958,bertozziInpaintingBinaryImages2007,lipkovaAutomatedUnsupervisedSegmentation2017}. Unlike conventional length-based or geometric regularization, which primarily smooths evolving interfaces~\cite{casellesGeodesicActiveContours1995,chanActiveContoursEdges2001}, the proposed phase-field mechanism actively promotes separated segmentation states and maintains interface discrimination when image-driven forces become weak or ambiguous.

From the analytical perspective, the nonlinear coupling of the softmax fidelity term, nonlocal region averages, and fourth-order phase-field regularization presents significant mathematical challenges. We derive the corresponding variational formulation together with a mixed $ L^2-H^{-1} $ gradient-flow system, establish the energy dissipation property, and prove the existence and uniqueness of weak solutions using Galerkin approximation, a priori estimates, compactness arguments, and suitable Lipschitz estimates for the softmax data force. The resulting weak solutions belong to the natural class \(C([0,T];L^2(\Omega))\cap L^2(0,T;H_N^2(\Omega))\), which reflects the regularity induced by the fourth-order phase-field term. From the numerical perspective, we develop a stabilized scalar auxiliary variable (SAV) splitting scheme~\cite{eyreUnconditionallyGradientStable1998,shenScalarAuxiliaryVariable2018,shenNewClassEfficient2019}, which leads to a linear implementation while preserving a modified discrete energy dissipation law. Extensive numerical experiments on weak-boundary, low-contrast, noisy, and medical images demonstrate the robustness of the proposed method in separating homogeneous structures and maintaining stable boundary localization.

The main contributions of this work are summarized as follows.

\begin{itemize}
\item We propose a smooth phase-separation variational model for weak-boundary
segmentation of homogeneous-appearance structures by integrating softmax-based
region fitting with Cahn--Hilliard phase-field regularization. A mixed
$ L^2-H^{-1} $ gradient flow couples data-driven evolution with higher-order
interfacial regularization, providing a physically motivated mechanism for
maintaining neighboring interfaces under weak image-driven forces.

\item We establish the mathematical and numerical foundations of the proposed
model. At the continuous level, we derive the associated energy dissipation law
and prove the existence and uniqueness of weak solutions. At the discrete level,
we develop a stabilized SAV splitting scheme that is linear at each time step
and satisfies a modified discrete energy dissipation law.

\item We validate the proposed method through extensive experiments on
weak-boundary, noisy, and medical image datasets. Comparisons with
representative variational, topology-preserving, phase-field, and deep learning
methods demonstrate that the proposed model can reliably separate adjacent
homogeneous structures while maintaining stable boundary localization.
\end{itemize}

The rest of this paper is organized as follows.
Section~\ref{sec:motivation} analyzes the limitations of level-set
segmentation for homogeneous structures and introduces the Cahn--Hilliard
phase-separation regularization.
Section~\ref{sec:model} presents the proposed softmax-based Cahn--Hilliard
variational model and its mixed gradient flow, and establishes energy
dissipation together with the existence and uniqueness of weak solutions.
Section~\ref{sec:numerical} develops the numerical method using energy
splitting and the scalar auxiliary variable approach, constructs a stabilized
linear scheme, and verifies its discrete energy stability.
Section~\ref{sec:experiments} evaluates boundary accuracy and stability through
regularization ablation, parameter sensitivity, noise robustness, comparisons
with variational and phase-field models, and experiments against deep learning
methods on the SKI10 dataset~\cite{heimannSegmentationKneeImages2010}.
Finally, Section~\ref{sec:conclusion} concludes the paper.

\section{Motivation and Background}
\label{sec:motivation}

\subsection{Limitations of Level-Set Segmentation for Homogeneous Structures}

The Chan--Vese (CV) model is one of the classical level-set based variational
methods for image segmentation~\cite{chanActiveContoursEdges2001}. Many
subsequent models can be viewed as extensions or modifications of the CV
framework, for example by replacing the global fitting energy with local
fitting terms, adding shape or topological constraints, or introducing
higher-order regularization. Therefore, the CV model provides a representative
example for analyzing why level-set segmentation may fail for
homogeneous-appearance but semantically distinct structures.

We first consider the two-phase CV model. Its data fitting energy can be
written in the level set formulation as
\begin{equation}\label{eq:cv-data-energy-related}
E_{\mathrm{data}}^{\mathrm{CV}}(\phi)
=
\lambda_1
\int_{\Omega}
(I-c_1)^2 H(\phi)\,dx
+
\lambda_2
\int_{\Omega}
(I-c_2)^2(1-H(\phi))\,dx,
\end{equation}
where \(I\) is the image intensity, \(H\) is the Heaviside function, and \(c_1,c_2\) 
are the mean intensities inside and outside the contour, respectively. 
Taking the variation with respect to \(\phi\), we obtain
\begin{equation}\label{eq:cv-data-variation-related}
\frac{\delta E_{\mathrm{data}}^{\mathrm{CV}}}{\delta \phi}
=
\delta(\phi)
\left[
\lambda_1(I-c_1)^2
-
\lambda_2(I-c_2)^2
\right].
\end{equation}
Therefore, the data-driven force in the negative gradient flow is

\begin{equation}\label{eq:cv-data-force-related}
F_{\mathrm{CV}}(x)
=
\delta(\phi)
\left[
-\lambda_1(I-c_1)^2
+
\lambda_2(I-c_2)^2
\right].
\end{equation}
The factor \(\delta(\phi)\) indicates that the force is concentrated in a narrow 
neighborhood of the zero level set. 
Thus, the interface evolution is determined by the contrast of the fitting errors 
near the current contour, rather than by the whole image domain.

In the special case \(\lambda_1=\lambda_2=1\), we have
\[
F_{\mathrm{CV}}(x)
=
\delta(\phi)
\left[
-(I-c_1)^2+(I-c_2)^2
\right].
\]
Since
\[
-(I-c_1)^2+(I-c_2)^2
=
(c_1-c_2)(2I-c_1-c_2),
\]
it follows that
\begin{equation}\label{eq:cv-force-factorized-related}
F_{\mathrm{CV}}(x)
=
\delta(\phi)(c_1-c_2)(2I-c_1-c_2).
\end{equation}
Hence, if the two fitted region intensities are close, namely
\[
c_1\approx c_2,
\]
then the effective force near the interface satisfies
\begin{equation}\label{eq:cv-force-degenerate-related}
F_{\mathrm{CV}}(x)\approx 0
\qquad
\text{for } x \text{ near } \{\phi=0\}.
\end{equation}
This factorized form shows the essential limitation of the CV data force. In
homogeneous-structure segmentation, two neighboring regions may have very
similar mean intensities, namely \(c_1\approx c_2\), even though they correspond
to different semantic or anatomical structures. In this case, the factor
\(c_1-c_2\) becomes small, and the CV data force may become too weak to drive
the contour toward the desired separating boundary.

To further illustrate this phenomenon, Fig.~\ref{fig:cv-force-visualization} shows 
a representative low-contrast segmentation result and the corresponding data-driven 
force of the CV model. 
Although the two adjacent structures should be separated, their local intensity
distributions are highly similar and the boundary between them is weak. The CV
segmentation tends to merge the neighboring regions, while the strong responses
of the data-force map are not necessarily aligned with the desired separating
interface. This indicates that the failure is not merely caused by poor
initialization or numerical error, but by the degeneration and ambiguity of the
data-driven force itself.

\begin{figure}[htbp]
	\centering
	\begin{subfigure}[t]{0.45\textwidth}
		\centering
		\includegraphics[
		width=\textwidth,
		height=0.28\textheight,
		keepaspectratio
		]{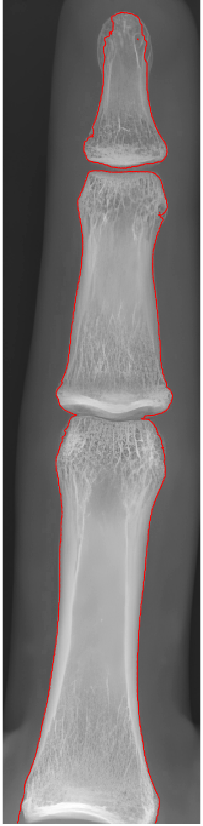}
		\caption{Segmentation result}
		\label{fig:CV_10095}
	\end{subfigure}
	\hfill
	\begin{subfigure}[t]{0.45\textwidth}
		\centering
		\includegraphics[
		width=\textwidth,
		height=0.28\textheight,
		keepaspectratio
		]{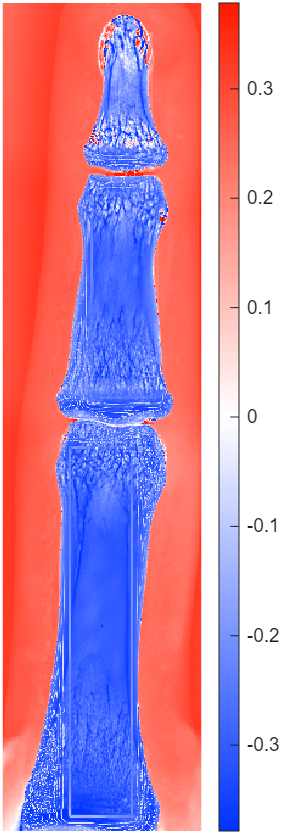}
		\caption{Data-driven force}
		\label{fig:CV_F_DATA}
	\end{subfigure}
	\caption{Segmentation result and data-driven force of the CV model in a low-contrast case.}
	\label{fig:cv-force-visualization}
\end{figure}

This observation is not limited to the original CV model. Many level-set based
extensions still rely on image-driven forces derived from intensity contrast,
local fitting statistics, edge indicators, or prescribed geometric constraints.
These modifications can improve robustness in many cases, but when the target
structures are homogeneous in appearance and separated only by weak boundaries,
the available image information may still be insufficient to distinguish the
two regions. Consequently, the evolving contour may leak, stop at an incorrect
position, or merge adjacent structures. This motivates the need for an
additional mechanism that does not merely smooth the contour, but promotes the
separation of different phases when the data force becomes weak.

\subsection{Cahn--Hilliard Phase-Separation Regularization}

To overcome the above limitation, we introduce the Cahn--Hilliard phase-field
regularization into the segmentation model~\cite{cahnFreeEnergyNonuniform1958,bertozziInpaintingBinaryImages2007,lipkovaAutomatedUnsupervisedSegmentation2017}.
The role of this term is different from conventional length or curvature
regularization. It not only smooths the evolving interface, but also provides a
phase-separation mechanism when the image-driven force becomes weak near the
target boundary.

For each phase variable \(u_a\), the Cahn--Hilliard regularization energy is defined as
\begin{equation}\label{eq:ch-energy-related}
E_{\mathrm{CH}}(u_a)
=
\int_{\Omega}
\left(
\frac{\varepsilon^2}{2}|\nabla u_a|^2
+
W(u_a)
\right)\,dx,
\end{equation}
where
\[
W(u_a)
=
\frac{1}{2}u_a^2(u_a-1)^2
\]
is a double-well potential. 
The gradient term penalizes sharp spatial oscillations and controls the thickness
of the diffuse interface, while the double-well potential encourages the phase
variable to approach the two preferred states \(0\) and \(1\). 
Therefore, the CH regularization is different from conventional length or curvature
regularization: it does not merely smooth the contour, but also promotes the
separation of different phases.

The variational derivative of \(E_{\mathrm{CH}}\) is
\begin{equation}\label{eq:ch-variation-related}
\frac{\delta E_{\mathrm{CH}}}{\delta u_a}
=
-\varepsilon^2\Delta u_a
+
W'(u_a).
\end{equation}
Thus, the phase-field part contributes the force
\[
-\varepsilon^2\Delta u_a+W'(u_a).
\]
When the data force is locally degenerated near a weak boundary, the evolution is
no longer determined only by the unreliable fitting-error contrast. 
Instead, the CH term continues to regulate the phase variables through the competition
between the interfacial energy and the double-well potential. 
In particular, even when
\[
F_{\mathrm{data}}\approx 0
\]
near the desired separating interface, the CH contribution may still be active unless
the phase variable has already reached a stable separated state. 
This provides an additional mechanism for pushing the phase functions toward
distinct labels while maintaining a regular diffuse interface.

Figure~\ref{fig:ch-force-visualization} further illustrates this effect. 
Compared with the CV results shown above~\cite{chanActiveContoursEdges2001}, the proposed CH-regularized model
successfully separates the adjacent structures around the low-contrast weak boundary. 
At the same time, the corresponding data-force map shows that \(F_{\mathrm{data}}\) is
very small, and nearly vanishes, near the target separating interface. 
This indicates that the successful separation is not mainly caused by a strong
image-driven force at the boundary. 
Rather, it is the CH phase-field regularization that supplies the missing
phase-separation mechanism when the fidelity force becomes insufficient. 
Therefore, the proposed model can achieve a separation that CV- and RSF-type models
fail to obtain under the same weak-boundary condition.

\begin{figure}[htbp]
	\centering
	\begin{subfigure}[t]{0.3\textwidth}
		\centering
		\includegraphics[
		width=\textwidth,
		height=0.28\textheight,
		keepaspectratio
		]{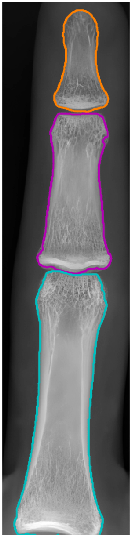}
		\caption{Segmentation result}
		\label{fig:CH_10095}
	\end{subfigure}
	\hfill
	\begin{subfigure}[t]{0.6\textwidth}
		\centering
		\includegraphics[
		width=\textwidth,
		height=0.28\textheight,
		keepaspectratio
		]{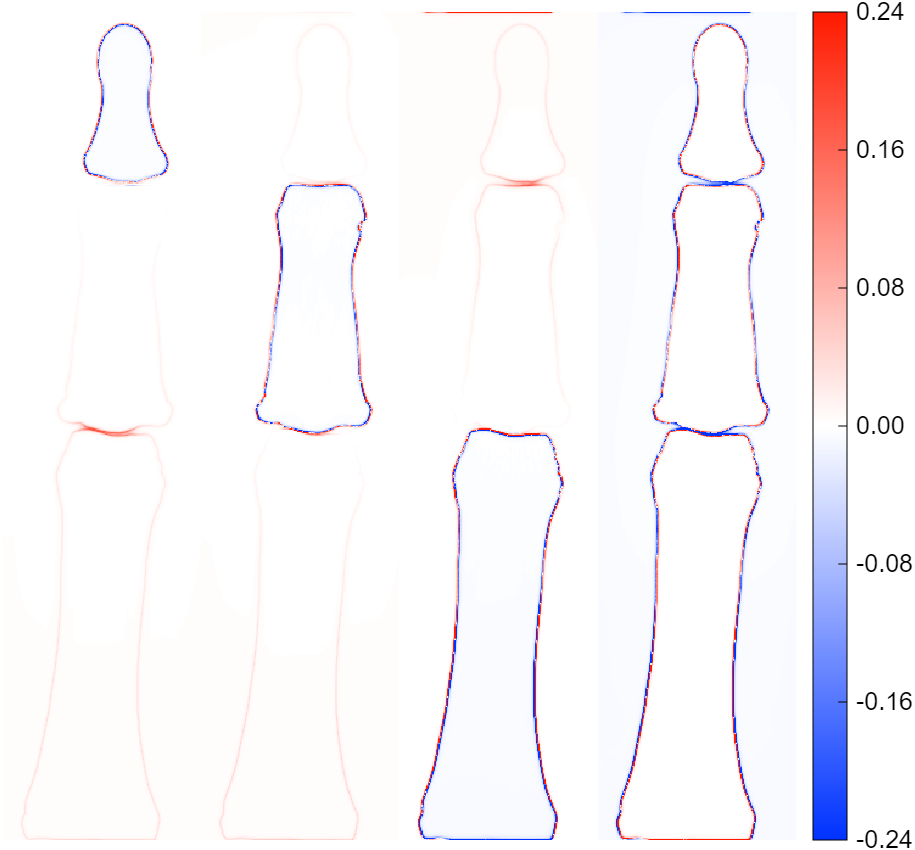}
		\caption{Data-driven force}
		\label{fig:CH_F_DATA}
	\end{subfigure}
	\caption{Segmentation result and data-driven force of the proposed CH-regularized model in a low-contrast case.}
	\label{fig:ch-force-visualization}
\end{figure}

These observations show that weak-boundary homogeneous-structure segmentation
requires an additional phase-separation mechanism when the image-driven force
becomes insufficient.


\section{Main method}
\label{sec:model}

Motivated by the analysis in Section~\ref{sec:motivation}, we develop a smooth
phase-separation variational model for weak-boundary segmentation of
homogeneous structures. The proposed model combines a smooth softmax-based
probabilistic representation with a distance-prior map \(D_a\) for target
localization, and uses the Cahn--Hilliard energy to provide phase-separation
regularization when the image-driven force is insufficient. The resulting
variational formulation and gradient flow are given below.\textsf{}

\subsection{Variational formulation}

In this section, we propose a multiphase image segmentation model that combines a soft classification mechanism with Cahn-Hilliard phase-field regularization. The softmax function maps the multichannel phase-field variables onto soft segmentation functions satisfying the simplex constraint. Meanwhile, a Cahn-Hilliard regularization term is introduced to describe phase separation and smooth interfacial evolution. This improves the stability and robustness of the model for weak-boundary and low-contrast images. The corresponding energy functional is defined as follows:
\begin{equation}\label{eq:total-energy}
	\begin{aligned}
		\E(U)
		&:=
		E_{\rm ch}(U)+E_{\rm data}(U)  \\
		&=
		\sum_{a=1}^K
		\int_\Omega
		\left(
		\frac{\varepsilon^2}{2}\abs{\nabla u_a}^2
		+W(u_a)
		\right)\dd x  
		+
		\lambda
		\sum_{a=1}^K
		\int_\Omega
		\Phi_a(U)
		\bigl((I-c_a)^2+\gamma_aD_a\bigr)\dd x .
	\end{aligned}
\end{equation}

Here, $U=(u_1,\ldots,u_K)$ denotes the multichannel phase-field variables, and
$W(u_a)=\frac{1}{2}u_a^2(u_a-1)^2$ is the double-well potential.
The term $D_a:\Omega\to[0,\infty)$ denotes a prescribed bounded
distance-penalty map associated with the $a$-th phase. It encodes spatial prior
information by assigning smaller values to pixels closer to the expected support
of phase $a$ and larger values to pixels farther away. Thus the combined fitting
cost $(I-c_a)^2+\gamma_aD_a$ balances intensity fidelity and distance-based
spatial preference, where $\gamma_a$ controls the strength of the distance
prior. In the following analysis, $D_a$ is treated as a fixed function in
$L^\infty(\Omega)$. Furthermore, the softmax function is defined as
\[
\Phi_a(U)
=
\frac{\exp(\beta u_a)}
{\sum_{b=1}^K \exp(\beta u_b)},
\qquad c_a=\frac{\int_\Omega I(x)\Phi_a(U)\mathrm{d}x}{\int_\Omega\Phi_a(U)\mathrm{d}x},\qquad a=1,\ldots,K,
\]
where $\beta>0$ controls the sharpness of the soft assignment.

The variational derivative of the energy functional with respect to $u_a$ is given by
\begin{equation}\label{eq:variational derivative}\begin{aligned}\frac{\delta \E(U)}{\delta u_a}=-\varepsilon^2\Delta u_a+W^{\prime}(u_a)+F_a(U),\end{aligned}\end{equation}
where
\[F_a(U)=
\lambda\beta\Phi_a(U)
\left(
e_a(c)-\bar{e}(c)
\right),\]
and
\[e_a(c)
=
(I-c_a)^2+\gamma_aD_a,
\qquad
\bar{e}(c)
=
\sum_{b=1}^K \Phi_b(U)e_b(c).
\]
Inspired by the work of Miyoun Jung and coauthors~\cite{jungSobolevGradientsJoint2009}, we further propose the following gradient flow:
\begin{equation}\label{eq:gradient flow}
	\begin{cases}\partial_tu_a=-(1-\Delta)\mu_a,\\\mu_a=-\varepsilon^2\Delta u_a+W^{\prime}(u_a)+F_a(U),&\end{cases}\quad a=1,\ldots,K,
\end{equation}
subject to the no-flux boundary conditions
\begin{equation}\label{eq:neumann-bc}
\frac{\partial u_a}{\partial\vec{n}}=0,
\qquad
\frac{\partial \mu_a}{\partial\vec{n}}=0
\quad\text{on }\partial\Omega\times(0,T),
\qquad a=1,\ldots,K .
\end{equation}

The final segmentation is obtained by converting the softmax phase functions into 
a hard label map. Specifically, after the gradient flow reaches a steady state, 
each pixel \(x\in\Omega\) is assigned to the phase with the largest softmax value:
\[
\ell(x)
=
\arg\max_{1\le a\le K}\Phi_a(U(x)).
\]
Therefore, \(\Phi_a(U)\) provides the soft membership function of the \(a\)-th phase, 
while \(\ell(x)\) gives the final discrete segmentation result.

\begin{remark}[Coupling direct descent with higher-order regularization]
The mixed $ L^2-H^{-1} $ flow
\[
\partial_t u_a=-(1-\Delta)\mu_a
\]
can be decomposed as
\[
\partial_t u_a=-\mu_a+\Delta\mu_a .
\]
The \(L^2\)-type component \(-\mu_a\) acts directly on the variational
derivative and enables an efficient response to the image fidelity force. The
Cahn--Hilliard-type component \(\Delta\mu_a\), on the other hand, diffuses the
chemical potential and introduces a fourth-order contribution through the term
\(-\varepsilon^2\Delta u_a\) in \(\mu_a\). The mixed dynamics therefore retain
direct data-driven motion while providing higher-order spatial regularization
for smoother and more coherent interface evolution.
\end{remark}

\begin{remark}[Adaptive evolution without mass conservation]
The classical \(H^{-1}\) flow \(\partial_tu_a=\Delta\mu_a\) conserves
\(\int_\Omega u_a\,\dd x\) under the homogeneous Neumann boundary condition. In
contrast, the proposed mixed flow satisfies
\[
\frac{\dd}{\dd t}\int_\Omega u_a\,\dd x
=
-\int_\Omega\mu_a\,\dd x,
\]
which is generally nonzero. Thus, the \(L^2\)-type component removes the strict
conservation of total phase mass and allows the phase content, and hence the
resulting region sizes, to adapt to the image data rather than being
predetermined by the initialization.
\end{remark}

\begin{remark}[Positive-definite mobility and spectral action]
Let \(\rho_j\ge0\) be an eigenvalue of \(-\Delta\) under the Neumann boundary
condition. The mobility multipliers associated with the \(L^2\) flow, the
\(H^{-1}\) flow, and the proposed mixed flow are, respectively,
\[
1,
\qquad
\rho_j,
\qquad
1+\rho_j .
\]
Hence, \(I-\Delta\) is self-adjoint and strictly positive, keeps the zero and
low-frequency modes active, and assigns larger descent weights to
high-frequency components of the chemical potential. It also yields the natural
dissipation
\[
\langle\mu_a,(I-\Delta)\mu_a\rangle
=
\|\mu_a\|_{L^2}^2+\|\nabla\mu_a\|_{L^2}^2,
\]
as established in the next subsection. Although a more general operator
\(\alpha I-\kappa\Delta\), with \(\alpha,\kappa>0\), could be used,
\(I-\Delta\) is the simplest normalized choice and avoids introducing an
additional relative mobility parameter.
\end{remark}

\subsection{Energy Dissipation}

\begin{proposition}[Energy dissipation]\label{prop:continuous-energy}
	Let \(U=(u_1,\ldots,u_K)\) be a sufficiently smooth solution of
	system~\eqref{eq:gradient flow} subject to the boundary
	conditions~\eqref{eq:neumann-bc}. Then the energy~\eqref{eq:total-energy} satisfies
	\begin{equation}\label{eq:continuous-energy-law}
		\frac{\dd}{\dd t}\E(U(t))
		=
		-\sum_{a=1}^K
		\left(
		\norm{\mu_a}_{L^2(\Omega)}^2
		+
		\norm{\nabla\mu_a}_{L^2(\Omega)}^2
		\right)
		\le 0.
	\end{equation}
\end{proposition}

\begin{proof}
	Using the variational derivative in
	\eqref{eq:variational derivative} and the definition of the chemical
	potential, we obtain
	\[
	\frac{\dd}{\dd t}\E(U(t))
	=
	\sum_{a=1}^K
	\ip{\frac{\delta\E}{\delta u_a}}{\partial_tu_a}
	=
	\sum_{a=1}^K\ip{\mu_a}{\partial_tu_a}.
	\]
	Substituting
	\(\partial_tu_a=-(1-\Delta)\mu_a\) and integrating by parts, using
	\(\frac{\partial \mu_a}{\partial\vec{n}}=0\), give
	\[
	\begin{aligned}
		\frac{\dd}{\dd t}\E(U(t))
		&=
		-\sum_{a=1}^K
		\ip{\mu_a}{(1-\Delta)\mu_a} \\
		&=
		-\sum_{a=1}^K
		\left(
		\norm{\mu_a}_{L^2(\Omega)}^2
		+
		\norm{\nabla\mu_a}_{L^2(\Omega)}^2
		\right).
	\end{aligned}
	\]
	This proves \eqref{eq:continuous-energy-law}.
\end{proof}

Consequently, for every \(t\ge0\),
\[
\E(U(t))
+
\sum_{a=1}^K
\int_0^t
\left(
\norm{\mu_a(s)}_{L^2(\Omega)}^2
+
\norm{\nabla\mu_a(s)}_{L^2(\Omega)}^2
\right)\dd s
=
\E(U(0)).
\]
The two dissipation terms reflect the mixed structure of the proposed flow. The
\(L^2\) term permits direct adjustment of the phase variables by the image data,
while the gradient term diffuses the chemical potential and regularizes the
interface.

\subsection{Theoretical Analysis}

\subsubsection{Existence of Weak Solutions}
Let $\Omega\subset\R^2$ be a bounded $C^2$ domain and let $T>0$.
Set $Q_T=\Omega\times(0,T).$
For an integer $K\ge2$, the unknown is
\[
U=(u_1,\dots,u_K):Q_T\to\R^K .
\]
For vector-valued functions we use the standard product $L^2$ norm
\[
\norm{U}_{L^2(\Omega)}^2
:=
\sum_{a=1}^K\norm{u_a}_{L^2(\Omega)}^2 .
\]
Define the space $V$ as $$
V:=H_N^2(\Omega)
:=
\{v\in H^2(\Omega):\frac{\partial v}{\partial\vec{n}}=0\text{ on }\partial\Omega\}.
$$
Then the weak solution of the gradient flow Equation \eqref{eq:gradient flow}  is defined as follows.

\begin{definition}[Weak solution]\label{def:weak_solution}
	A vector field $U=(u_1,\dots,u_K)$ is a weak solution of
	\eqref{eq:gradient flow} on $[0,T]$ if, for every $a=1,\dots,K$,
	\[
	u_a\in C([0,T];L^2(\Omega))
	\cap L^2(0,T;V)
	,
	\qquad
	\partial_tu_a\in L^2(0,T;V'),
	\]
	and, for every $v\in V$ and a.e. $t\in(0,T)$,
	\begin{align}
		\dual{\partial_tu_a}{v}
		&+\varepsilon^2\ip{\Delta u_a}{\Delta v}
		+\varepsilon^2\ip{\nabla u_a}{\nabla v}
		-\ip{W'(u_a)}{\Delta v}
		+\ip{W'(u_a)}{v}        \notag\\
		&-\ip{F_a(U)}{\Delta v}
		+\ip{F_a(U)}{v}=0.      \label{eq:weak-form}
	\end{align}
	Moreover $u_a(0)=u_{a,0}$ in $L^2(\Omega)$.
\end{definition}

To prove the existence of weak solution, we introduce some auxiliary lemmas.
\begin{lemma}[Softmax bounds]\label{lem:softmax}
	For every $U,V\in\R^K$,
	\[
	0<\Phi_a(U)<1,\qquad \sum_{a=1}^K\Phi_a(U)=1,
	\]
	and there exists $C_\beta>0$ such that
	\[
	\abs{\Phi(U)-\Phi(V)}
	\le C_\beta\abs{U-V}.
	\]
	Consequently,
	\[
	\norm{\Phi(U)-\Phi(V)}_{L^2(\Omega)}
	\le C_\beta\norm{U-V}_{L^2(\Omega)} .
	\]
\end{lemma}

\begin{proof}
	The Jacobian of the softmax is
	\[
	\frac{\partial\Phi_a}{\partial u_b}
	=
	\beta\Phi_a(\delta_{ab}-\Phi_b),
	\]
	which is uniformly bounded on $\R^K$.  The mean value theorem gives the
	global Lipschitz bound.
\end{proof}

\begin{lemma}[Uniform lower bound for the softmax mass]\label{lem:mass}
	Assume
	\[
	\sup_{t\in[0,T]}\sum_{b=1}^K\norm{u_b(t)}_{L^2(\Omega)}^2\le M_T.
	\]
	Then there exists $m_*(T)>0$, depending only on
	$\Omega,\beta,K,M_T$, such that
	\[
	m_a(U)(t):=\int_\Omega\Phi_a(U(x,t))\dd x
	\ge m_*(T)
	\]
	for every $a=1,\dots,K$ and every $t\in[0,T]$.
\end{lemma}

\begin{proof}
	For fixed $a$,
	\[
	\Phi_a(U)
	=\frac{e^{\beta u_a}}{\sum_{b=1}^K e^{\beta u_b}}
	\ge
	\frac1K\exp\left(-2\beta\sum_{b=1}^K\abs{u_b}\right).
	\]
	By Jensen's inequality,
	\[
	\begin{aligned}
		m_a(U)(t)
		&\ge
		\frac1K\int_\Omega
		\exp\left(-2\beta\sum_{b=1}^K\abs{u_b(x,t)}\right)\dd x\\
		&\ge
		\frac{|\Omega|}{K}
		\exp\left(
		-\frac{2\beta}{|\Omega|}
		\int_\Omega\sum_{b=1}^K\abs{u_b(x,t)}\dd x
		\right).
	\end{aligned}
	\]
	Cauchy's inequality gives
	\[
	\int_\Omega\sum_{b=1}^K\abs{u_b}
	\le |\Omega|^{1/2}\sqrt K
	\left(\sum_{b=1}^K\norm{u_b}_{L^2}^2\right)^{1/2}
	\le |\Omega|^{1/2}\sqrt K\,M_T^{1/2}.
	\]
	Thus
	\[
	m_a(U)(t)\ge
	\frac{|\Omega|}{K}
	\exp\left(
	-\frac{2\beta\sqrt K}{|\Omega|^{1/2}}M_T^{1/2}
	\right)
	=:m_*(T)>0 .
	\]
\end{proof}

\begin{lemma}[Bounds and Lipschitz control for the data force]\label{lem:F}
	For every admissible $U$ and every $a$,
	\[
	\abs{c_a(U)(t)}\le \norm{I}_{L^\infty(\Omega)}.
	\]
	Hence there is a constant $C_F>0$, depending only on the fixed data and
	parameters, such that
	\[
	\abs{F_a(U,x,t)}\le C_F
	\quad\text{a.e. in }Q_T.
	\]
	Moreover, if $U,V$ both satisfy the bound in Lemma~\ref{lem:mass}, then
	there is $C_T>0$ such that
	\[
	\abs{c_a(U)(t)-c_a(V)(t)}
	\le C_T\norm{U(t)-V(t)}_{L^2(\Omega)}
	\]
	and
	\[
	\norm{F(U)(t)-F(V)(t)}_{L^2(\Omega)}
	\le C_T\norm{U(t)-V(t)}_{L^2(\Omega)} .
	\]
\end{lemma}

\begin{proof}
	Since $\Phi_a\ge0$ and $m_a(U)>0$,
	\[
	\abs{c_a(U)}
	\le
	\frac{\norm{I}_\infty\int_\Omega\Phi_a(U)\dd x}
	{\int_\Omega\Phi_a(U)\dd x}
	=\norm{I}_\infty .
	\]
	It follows that
	\[
	\abs{e_a(U)}
	\le 4\norm{I}_\infty^2
	+\max_b\abs{\gamma_b}\max_b\norm{D_b}_\infty
	=:C_e,
	\]
	and therefore $\abs{\bar e(U)}\le C_e$ and
	\[
	\abs{F_a(U)}\le 2\lambda\beta C_e=:C_F .
	\]
	
	For the Lipschitz part, write
	\[
	N_a(U)=\int_\Omega I\Phi_a(U)\dd x,\qquad
	m_a(U)=\int_\Omega\Phi_a(U)\dd x .
	\]
	Using Lemma~\ref{lem:softmax} and Lemma~\ref{lem:mass},
	\[
	\begin{aligned}
		\abs{c_a(U)-c_a(V)}
		&\le
		\frac{\abs{N_a(U)-N_a(V)}}{m_a(U)}
		+\abs{N_a(V)}
		\frac{\abs{m_a(U)-m_a(V)}}{m_a(U)m_a(V)}\\
		&\le
		C_T\norm{U-V}_{L^2(\Omega)} .
	\end{aligned}
	\]
	Consequently
	\[
	\norm{e_a(U)-e_a(V)}_{L^2}
	\le C_T\norm{U-V}_{L^2(\Omega)}.
	\]
	Combining this with the softmax Lipschitz bound in
	\[
	\bar e(U)-\bar e(V)
	=
	\sum_b(\Phi_b(U)-\Phi_b(V))e_b(U)
	+\sum_b\Phi_b(V)(e_b(U)-e_b(V))
	\]
	gives the same $L^2$ control for $\bar e(U)-\bar e(V)$.
	The definition of $F_a$ then yields the asserted Lipschitz estimate.
\end{proof}

Next, we give the existence theorem of the weak solution.

\begin{theorem}[Existence of weak solutions]\label{thm:existence}
	Assume
	\[
	I,D_a\in L^\infty(\Omega),\qquad
	\gamma_a\in\R,\qquad
	\varepsilon,\beta,\lambda>0,\qquad
	u_{a,0}\in L^2(\Omega).
	\]
	Then, for every $T>0$, system~\eqref{eq:gradient flow} admits a weak
	solution $U=(u_1,\dots,u_K)$ in the sense of Definition~\ref{def:weak_solution}.
\end{theorem}

\begin{proof}
	\textbf{Step 1: Galerkin approximation.}
	Let $\{w_j\}_{j=1}^\infty$ be an $L^2$-orthonormal basis of $V$ and set
	$V_m=\operatorname{span}\{w_1,\dots,w_m\}$.  Seek
	\[
	u_a^m(x,t)=\sum_{j=1}^m g_{aj}^m(t)w_j(x),
	\qquad a=1,\dots,K,
	\]
	such that, for every $v\in V_m$,
	\begin{align}
		\ip{\partial_tu_a^m}{v}
		&+\varepsilon^2\ip{\Delta u_a^m}{\Delta v}
		+\varepsilon^2\ip{\nabla u_a^m}{\nabla v}
		-\ip{W'(u_a^m)}{\Delta v}
		+\ip{W'(u_a^m)}{v}        \notag\\
		&-\ip{F_a(U^m)}{\Delta v}
		+\ip{F_a(U^m)}{v}=0,      \label{eq:galerkin}
	\end{align}
	with $u_a^m(0)=P_m u_{a,0}$.  In finite dimension this is a locally Lipschitz
	ODE system because $W'$ is polynomial and the softmax denominator is strictly
	positive on bounded coefficient sets.  Hence there exists a unique solution of the corresponding ODE system.

	\textbf{Step 2: uniform $L^\infty_tL^2_x$ and $L^2_tH^2_x$ estimates.}
	Taking $v=u_a^m$ in \eqref{eq:galerkin} gives
	\[
	\begin{aligned}
		\frac12\frac{\dd}{\dd t}\norm{u_a^m}_{L^2}^2
		&+\varepsilon^2\norm{\Delta u_a^m}_{L^2}^2
		+\varepsilon^2\norm{\nabla u_a^m}_{L^2}^2       \\
		&=
		\ip{W'(u_a^m)}{\Delta u_a^m}
		-\ip{W'(u_a^m)}{u_a^m}
		+\ip{F_a(U^m)}{\Delta u_a^m}
		-\ip{F_a(U^m)}{u_a^m}.
	\end{aligned}
	\]
	Since
	\[
	W''(s)=6s^2-6s+1\ge-\frac12,
	\]
	integration by parts yields
	\[
	\ip{W'(u_a^m)}{\Delta u_a^m}
	=
	-\int_\Omega W''(u_a^m)\abs{\nabla u_a^m}^2\dd x
	\le \frac12\norm{\nabla u_a^m}_{L^2}^2.
	\]
	By interpolation,
	\[
	\norm{\nabla v}_{L^2}^2
	\le \delta\norm{\Delta v}_{L^2}^2+C_\delta\norm{v}_{L^2}^2 .
	\]
	Also $-sW'(s)=-2s^4+3s^3-s^2\le C$ for all $s\in\R$.  Lemma~\ref{lem:F}
	gives
	\[
	\abs{\ip{F_a(U^m)}{\Delta u_a^m}}
	\le \delta\norm{\Delta u_a^m}_{L^2}^2+C_\delta,
	\qquad
	\abs{\ip{F_a(U^m)}{u_a^m}}
	\le C(1+\norm{u_a^m}_{L^2}^2).
	\]
	Choosing $\delta>0$ small enough and summing over $a$,
	\[
	Y_m'(t)+c_0\sum_{a=1}^K\norm{\Delta u_a^m}_{L^2}^2
	\le C(1+Y_m(t)),
	\qquad
	Y_m(t)=\sum_{a=1}^K\norm{u_a^m(t)}_{L^2}^2 .
	\]
	Gronwall's inequality gives
	\[
	\sup_{0\le t\le T}Y_m(t)\le C_T,
	\qquad
	\sum_{a=1}^K\int_0^T\norm{\Delta u_a^m}_{L^2}^2\dd t\le C_T.
	\]
	The Neumann elliptic estimate
	\[
	\norm{v}_{H^2}\le C\bigl(\norm{v}_{L^2}+\norm{\Delta v}_{L^2}\bigr)
	\]
	then implies uniform boundedness in
	$L^\infty(0,T;L^2)\cap L^2(0,T;H^2)$.
	
	Consequently, by weak compactness and the weak closedness of
	$V=H_N^2(\Omega)$, any Galerkin limit satisfies
	\[
	u_a^m\in L^\infty(0,T;L^2(\Omega))\cap L^2(0,T;V),
	\]
	which gives the main spatial part of the desired weak-solution class.
	
	\textbf{Step 3: higher integrability and the nonlinear potential.}
	In two dimensions, $v$ satisfies homogeneous Neumann boundary conditions
	\[
	\norm{v}_{H^1}^4
	\le C_{\rm GN}\norm{v}_{L^2}^2\norm{v}_{H^2}^2,
	\]
	hence, using Step 2,
	\[
	\int_0^T\norm{u_a^m(t)}_{H^1}^4\dd t
	\le
	C_{\rm GN}
	\left(\sup_{0\le t\le T}\norm{u_a^m(t)}_{L^2}^2\right)
	\int_0^T\norm{u_a^m(t)}_{H^2}^2\dd t
	\le C_{T,1}.
	\]
	Thus $u_a^m$ is uniformly bounded in $L^4(0,T;H^1(\Omega))$.
	
	Also,because of Gagliardo-Nirenberg inequality,
	\[
	\norm{v}_{L^6}^6\le C_6\norm{v}_{H^2}^2\norm{v}_{L^2}^4,
	\qquad
	\abs{W'(s)}\le C_W(1+\abs{s}^3),
	\]
	and therefore
	\[
	\begin{aligned}
		\int_0^T\norm{W'(u_a^m)}_{L^2}^2\dd t
		&\le
		C_W'\int_0^T\bigl(1+\norm{u_a^m}_{L^6}^6\bigr)\dd t  \\
		&\le
		C_{T,2}
		+C_6C_W'
		\left(\sup_{0\le t\le T}\norm{u_a^m(t)}_{L^2}^4\right)
		\int_0^T\norm{u_a^m(t)}_{H^2}^2\dd t
		\le C_{T,3}.
	\end{aligned}
	\]
	So $W'(u_a^m)$ is uniformly bounded in $L^2(0,T;L^2(\Omega))$.
	
	\textbf{Step 4: time derivative estimate.}
	For $v\in V$ with $\norm{v}_V\le1$, \eqref{eq:galerkin} gives
	\[
	\begin{aligned}
		\abs{\dual{\partial_tu_a^m}{v}}
		&\le
		\varepsilon^2\norm{\Delta u_a^m}_{L^2}\norm{\Delta v}_{L^2}
		+\varepsilon^2\norm{\nabla u_a^m}_{L^2}\norm{\nabla v}_{L^2} \\
		&\quad
		+\norm{W'(u_a^m)}_{L^2}
		\bigl(\norm{\Delta v}_{L^2}+\norm{v}_{L^2}\bigr)
		+\norm{F_a(U^m)}_{L^2}
		\bigl(\norm{\Delta v}_{L^2}+\norm{v}_{L^2}\bigr) .
	\end{aligned}
	\]
	Taking the supremum over $\norm{v}_V\le1$ 
	\[ \abs{\dual{\partial_tu_a^m}{v}}\le C_*
	\bigl(\norm{u_a^m}_{H^2}
	+\norm{W'(u_a^m)}_{L^2}
	+1\bigr).
	\]
	Integrating over \((0,T)\), we obtain
	\[
	\int_0^T\norm{\partial_tu_a^m(t)}_{V'}^2\dd t
	\le
	C_*'
	\int_0^T
	\left(
	\norm{u_a^m}_{H^2}^2
	+\norm{W'(u_a^m)}_{L^2}^2
	+1
	\right)\dd t
	\le C_{T,4}.
	\]
	Therefore $\partial_tu_a^m$ is uniformly bounded in $L^2(0,T;V')$.
	
	\textbf{Step 5: compactness and passage to the limit.}
	Aubin--Lions gives, after extracting a subsequence,
	\[
	u_a^m\to u_a\quad\text{strongly in }L^2(0,T;H^1(\Omega)),
	\]
	\[
	u_a^m\rightharpoonup u_a\quad\text{weakly in }L^2(0,T;H^2(\Omega)),
	\qquad
	u_a^m\overset{*}{\rightharpoonup}u_a
	\quad\text{in }L^\infty(0,T;L^2(\Omega)),
	\]
	and
	\[
	\partial_tu_a^m\rightharpoonup\partial_tu_a
	\quad\text{weakly in }L^2(0,T;V').
	\]
	The uniform $L^\infty_tL^2_x$ estimate and Lemma~\ref{lem:mass} give a
	positive lower bound for
	$\int_\Omega\Phi_a(U^m)\dd x$, independent of $m$.  Lemma~\ref{lem:softmax}
	then yields $\Phi(U^m)\to\Phi(U)$ strongly in $L^2(Q_T)$.
	Consequently the numerators and denominators in $c_a(U^m)$ converge strongly
	in $L^2(0,T)$, and the positive lower bound allows division:
	\[
	c_a(U^m)\to c_a(U)\quad\text{strongly in }L^2(0,T).
	\]
	It follows that
	\[
	F_a(U^m)\to F_a(U)\quad\text{strongly in }L^2(Q_T).
	\]
	The polynomial term satisfies
	$W'(u_a^m)\rightharpoonup W'(u_a)$ weakly in $L^2(Q_T)$.  Indeed,
	the strong convergence in $L^2(0,T;H^1)$ gives a.e. convergence along a
	subsequence, while the uniform $L^2(0,T;L^6)$ bound controls the cubic growth
	of $W'$.  Hence all terms in \eqref{eq:galerkin} pass to the limit, and $U$
	satisfies \eqref{eq:weak-form}. 
	Hence,we gain $$\begin{aligned}\langle\partial_{t}u_{a},v\rangle&+\varepsilon^{2}\left(\Delta u_{a},\Delta v\right)+\varepsilon^{2}\left(\nabla u_{a},\nabla v\right)-\left(W^{\prime}(u_{a}),\Delta v\right)+\left(W^{\prime}(u_{a}),v\right)\\&-\left(F_{a}(U),\Delta v\right)+\left(F_{a}(U),v\right)=0.\end{aligned}$$
	
	 Finally,
	\[
	u_a\in L^2(0,T;V),\qquad \partial_tu_a\in L^2(0,T;V')
	\]
	imply $u_a\in C([0,T];L^2(\Omega))$ by the Lions--Magenes lemma, and
	$u_a(0)=u_{a,0}$ follows from $P_m u_{a,0}\to u_{a,0}$ in $L^2$.
\end{proof}

\subsubsection{Uniqueness}

\begin{theorem}[Uniqueness]\label{}
	In the class of weak solutions described above, the solution  in the theorem
	\eqref{thm:existence} is unique.
\end{theorem}

\begin{proof}
	Let $U=(u_1,\dots,u_K)$ and $V=(v_1,\dots,v_K)$ be two weak solutions with
	the same initial data.  Set
	\[
	w_a=u_a-v_a,\qquad W=U-V.
	\]
	In this proof, $\norm{W}_{L^2}$ always denotes the product norm
	\[
	\norm{W}_{L^2}^2
	=
	\sum_{a=1}^K\norm{w_a}_{L^2(\Omega)}^2,
	\]
	and similarly for $\norm{\nabla W}_{L^2}$ and $\norm{F(U)-F(V)}_{L^2}$.
	Let $\{\varphi_j\}_{j=0}^\infty$ be an $L^2$-orthonormal Neumann eigenbasis,
	\[
	-\Delta\varphi_j=\lambda_j\varphi_j,\qquad \frac{\partial \varphi_j}{\partial\vec{n}}=0,
	\]
	and define
	\[
	\Nop=(I-\Delta_N)^{-1}.
	\]
	Here $\Delta_N$ denotes the Laplacian with homogeneous Neumann boundary
	condition; equivalently, for $g=\Nop f$,
	\[
	g-\Delta g=f\quad\text{in }\Omega,
	\qquad
	\frac{\partial g}{\partial\vec{n}}=0 \quad\text{on }\partial\Omega.
	\]
	Hence $\Nop f\in H_N^2(\Omega)=V$ whenever $f\in L^2(\Omega)$.
	If $f=\sum_jc_j\varphi_j$, set
	\[
	\norm{f}_*^2=(\Nop f,f)
	=
	\sum_{j=0}^\infty\frac{\abs{c_j}^2}{1+\lambda_j}.
	\]
	
	Subtracting the weak formulations for $u_a$ and $v_a$ gives, for every
	$z\in V$,
	\begin{align}
		\dual{\partial_tw_a}{z}
		&+\varepsilon^2\ip{\Delta w_a}{\Delta z}
		+\varepsilon^2\ip{\nabla w_a}{\nabla z} \notag\\
		&-\ip{W'(u_a)-W'(v_a)}{\Delta z}
		+\ip{W'(u_a)-W'(v_a)}{z} \notag\\
		&-\ip{F_a(U)-F_a(V)}{\Delta z}
		+\ip{F_a(U)-F_a(V)}{z}=0 .
		\label{eq:weak-difference}
	\end{align}
	Testing \eqref{eq:weak-difference} with $z=\Nop w_a\in V$ and using
	$(I-\Delta)\Nop w_a=w_a$ gives
	\[
	\frac12\frac{\dd}{\dd t}\norm{w_a}_*^2
	+\varepsilon^2\norm{\nabla w_a}_{L^2}^2
	+\ip{W'(u_a)-W'(v_a)}{w_a}
	+\ip{F_a(U)-F_a(V)}{w_a}
	=0.
	\]
	Here the identities are direct consequences of the Neumann spectral
	decomposition.  If $w_a=\sum_j c_j\varphi_j$, then
	\[
	\ip{\Delta w_a}{\Delta\Nop w_a}
	+\ip{\nabla w_a}{\nabla\Nop w_a}
	=
	\sum_j\lambda_j\abs{c_j}^2
	=
	\norm{\nabla w_a}_{L^2}^2,
	\]
	and, for any scalar field $G$,
	\[
	-\ip{G}{\Delta\Nop w_a}+\ip{G}{\Nop w_a}
	=
	\ip{G}{(I-\Delta)\Nop w_a}
	=
	\ip{G}{w_a}.
	\]
	After summing in $a$,
	\begin{align}
		\frac12\frac{\dd}{\dd t}\norm{W}_*^2
		&+\varepsilon^2\norm{\nabla W}_{L^2}^2
		+\sum_{a=1}^K\ip{W'(u_a)-W'(v_a)}{w_a} \notag\\
		&+\sum_{a=1}^K\ip{F_a(U)-F_a(V)}{w_a}=0. \label{eq:unique-energy}
	\end{align}
	Since $W''(s)=6s^2-6s+1\ge-\frac12$, the mean value theorem gives
	\[
	\bigl(W'(r)-W'(s)\bigr)(r-s)\ge-\frac12\abs{r-s}^2,
	\]
	hence
	\[
	\sum_{a=1}^K\ip{W'(u_a)-W'(v_a)}{w_a}
	\ge -\frac12\norm{W}_{L^2}^2.
	\]
	Both weak solutions satisfy a common $L^\infty_tL^2_x$ bound, so
	Lemma~\ref{lem:F} applies:
	\[
	\norm{F(U)-F(V)}_{L^2}
	\le C_T\norm{W}_{L^2} .
	\]
	Thus
	\[
	\left|
	\sum_{a=1}^K\ip{F_a(U)-F_a(V)}{w_a}
	\right|
	\le C_T\norm{W}_{L^2}^2.
	\]
	Equation~\eqref{eq:unique-energy} becomes
	\[
	\frac12\frac{\dd}{\dd t}\norm{W}_*^2
	+\varepsilon^2\norm{\nabla W}_{L^2}^2
	\le C_T\norm{W}_{L^2}^2.
	\]
	For every $\delta>0$ there is $C_\delta>0$ such that
	\[
	1\le \delta\lambda+\frac{C_\delta}{1+\lambda},
	\qquad \lambda\ge0.
	\]
	Applying this spectral inequality to each component gives
	\[
	\norm{W}_{L^2}^2
	\le
	\delta\norm{\nabla W}_{L^2}^2
	+C_\delta\norm{W}_*^2.
	\]
	Choosing $\delta$ small enough to absorb the gradient term, we obtain
	\[
	\frac{\dd}{\dd t}\norm{W}_*^2
	\le C_T\norm{W}_*^2.
	\]
	Since $W(0)=0$, Gronwall's inequality gives
	\[
	\norm{W(t)}_*^2=0,\qquad 0\le t\le T.
	\]
	Therefore $U=V$.
\end{proof}


\section{Numerical algorithms}
\label{sec:numerical}

The proposed model is numerically solved through its gradient flow. For Cahn--Hilliard-type high-order nonlinear problems~\cite{eyreUnconditionallyGradientStable1998,zhuCoarseningKineticsVariablemobility1999}, explicit schemes are severely restricted by stability conditions, whereas fully implicit schemes are costly due to the nonlinear solves required at each time step. Related numerical methods have also been developed for phase-field and free-energy models~\cite{jiangEfficientNumericalMethods2020,baoConvergenceAnalysisBregman2024}. To balance stability and efficiency, we develop a linear structure-preserving scheme by combining stabilized energy splitting with the SAV approach. The scheme is linearly solvable and unconditionally energy stable, thereby improving both numerical stability and computational efficiency.

\subsection{Energy Splitting Method}

Energy splitting is a classical technique for constructing stable numerical schemes for gradient flows~\cite{eyreUnconditionallyGradientStable1998}. The basic idea is to decompose the total energy into two parts,
\[
E(U)=E_c(U)-E_e(U),
\]
where the convex or stabilizing part \(E_c\) is treated implicitly, while the remaining part \(E_e\) is treated explicitly. A typical semi-implicit scheme takes the form
\[
\frac{U^{n+1}-U^n}{\tau}
=
-\mathcal{M}
\left(
\frac{\delta E_c}{\delta U}(U^{n+1})
-
\frac{\delta E_e}{\delta U}(U^n)
\right),
\]
where \(\mathcal{M}\) is the mobility operator.

The main advantage of this approach is its clear variational structure. By treating the stiff linear or convex part implicitly, the scheme can achieve good numerical stability and avoid the severe time-step restriction of fully explicit methods. In particular, for Cahn--Hilliard type equations, the high-order linear term can be handled implicitly to reduce the stiffness caused by the fourth-order diffusion operator.

However, a pure energy splitting strategy may become difficult when the nonlinear energy has a complicated structure. In the present model, the data term contains the softmax functions and the region averages \(c_a(U)\), which depend nonlinearly on \(U\). Therefore, it is not straightforward to construct a strict convex--concave splitting for the whole energy. In practice, one may need a large stabilization parameter or a stronger linearization of the nonlinear data term, which may affect the accuracy and efficiency of the method.

\subsection{Scalar Auxiliary Variable Method}

The scalar auxiliary variable (SAV) method provides another effective way to construct stable linear schemes for nonlinear gradient flows~\cite{shenScalarAuxiliaryVariable2018,shenNewClassEfficient2019}. Suppose the energy can be written as
\[
E(U)
=
\frac{1}{2}
\sum_{a=1}^K
\left\langle
\mathcal{L}u_a,u_a
\right\rangle
+
E_1(U),
\]
where \(\mathcal{L}\) is a linear positive definite operator and \(E_1(U)\) contains the nonlinear part of the energy. The SAV method introduces a scalar auxiliary variable
\[
r(t)=\sqrt{E_1(U(t))+C_0},
\]
where \(C_0>0\) is chosen so that \(E_1(U)+C_0>0\). Then the nonlinear contribution can be rewritten in a form suitable for linear time discretization.

The advantage of the SAV method is that it does not require a convex--concave decomposition of \(E_1(U)\). Therefore, it is particularly suitable for complicated nonlinear energies, such as the softmax-based fidelity term in the proposed model. Moreover, after introducing the auxiliary variable, the resulting numerical scheme is usually linear and can preserve a modified discrete energy dissipation law.

Nevertheless, the SAV method also has some limitations. The dissipated energy is generally a modified energy involving the auxiliary variable, rather than exactly the original physical energy. Hence, the original energy may exhibit small oscillations in numerical simulations, especially when a relatively large time step is used. In such cases, the original energy is not guaranteed to be monotonically decreasing and may become unstable due to the explicit treatment of nonlinear terms.

\subsection{Proposed Stabilized SAV Splitting Scheme}

To address the shortcomings of the two methods in the segmentation model, we propose a stabilized SAV splitting scheme for the present model. 

By introducing a stabilization parameter \(\alpha>0\), the chemical potential is split as
\[
\begin{aligned}
	\mu_a
	&=
	\underbrace{
		\left(
		-\varepsilon^2\Delta u_a+\alpha u_a
		\right)
	}_{\text{linear implicit part}}
	+
	\underbrace{
		\left(
		W'(u_a)-\alpha u_a+F_a(U)
		\right)
	}_{\text{nonlinear SAV part}}                                      \\[2mm]
	&=
	\mathcal{L}u_a+B_a(U),
\end{aligned}
\]

Accordingly, the energy can be rewritten as
\[
E(U)
=
\frac{1}{2}
\sum_{a=1}^K
\left\langle
\mathcal{L}u_a,u_a
\right\rangle
+
E_1(U),
\]
where
\[
E_1(U)
=
\sum_{a=1}^K
\int_{\Omega}
\left[
W(u_a)-\frac{\alpha}{2}u_a^2
\right]\,dx
+
\lambda
\sum_{a=1}^K
\int_{\Omega}
\Phi_a(U)e_a(U)\,dx.
\]
Choose a constant \(C_0>0\) such that \(E_1(U)+C_0>0\), and define the scalar auxiliary variable
\[
r(t)
=
\sqrt{E_1(U)+C_0}.
\]
Then the stabilized SAV reformulation of the gradient flow is written as
\begin{subequations}\label{eq:sav-continuous}
	\begin{align}
		\partial_t u_a
		&=
		-\mathcal{M}\mu_a,
		\label{eq:sav-continuous-a}
		\\
		\mu_a
		&=
		\mathcal{L}u_a
		+
		\frac{r}{\sqrt{E_1(U)+C_0}}
		B_a(U),
		\label{eq:sav-continuous-b}
		\\
		\frac{d r}{dt}
		&=
		\frac{1}{2\sqrt{E_1(U)+C_0}}
		\sum_{a=1}^K
		\left\langle
		B_a(U),
		\partial_t u_a
		\right\rangle .
		\label{eq:sav-continuous-c}
	\end{align}
\end{subequations}
Here the mobility operator is taken as
\[
\mathcal{M}=I-\Delta .
\]

At each time step, we first compute $c^n$,and then we get
\[
B_a^n
=
W'(u_a^n)-\alpha u_a^n+F_a(U^n),
\qquad
q_a^n
=
\frac{B_a^n}
{\sqrt{E_1(U^n)+C_0}}.
\]
The first-order SAV time discretization is given by
\begin{subequations}\label{eq:sav-discrete}
	\begin{align}
		\frac{u_a^{n+1}-u_a^n}{\tau}
		&=
		-\mathcal{M}\mu_a^{n+1},
		\label{eq:sav-discrete-a}
		\\
		\mu_a^{n+1}
		&=
		\mathcal{L}u_a^{n+1}
		+
		r^{n+1}q_a^n,
		\label{eq:sav-discrete-b}
		\\
		\frac{r^{n+1}-r^n}{\tau}
		&=
		\frac{1}{2}
		\sum_{a=1}^K
		\left\langle
		q_a^n,
		\frac{u_a^{n+1}-u_a^n}{\tau}
		\right\rangle .
		\label{eq:sav-discrete-c}
	\end{align}
\end{subequations}

Substituting \eqref{eq:sav-discrete-b} into \eqref{eq:sav-discrete-a}, we obtain
\[
A u_a^{n+1}
+
\tau\,\mathcal{M}q_a^n\,r^{n+1}
=
u_a^n,
\qquad
A
=
I+\tau\,\mathcal{M}\mathcal{L}.
\]
For efficient implementation, we introduce
\[
\begin{aligned}
	s_a^n
	&=
	A^{-1}u_a^n,
	\\
	t_a^n
	&=
	A^{-1}
	\left(
	\tau\,\mathcal{M}q_a^n
	\right),
	\\
	u_a^{n+1}
	&=
	s_a^n-r^{n+1}t_a^n,
	\\
	r^{n+1}
	&=
	\frac{
		r^n
		+
		\frac{1}{2}
		\sum_{a=1}^K
		\left\langle
		q_a^n,
		s_a^n-u_a^n
		\right\rangle
	}{
		1
		+
		\frac{1}{2}
		\sum_{a=1}^K
		\left\langle
		q_a^n,
		t_a^n
		\right\rangle
	}.
\end{aligned}
\]

Therefore, for each phase variable, only two linear systems with the same coefficient operator \(A\) need to be solved at each time step. Under the Fourier spectral discretization, the operator \(A\) becomes diagonal in Fourier space. Since
\[
-\Delta
\quad\longleftrightarrow\quad
|k|^2,
\]
we have
\[
\widehat{\mathcal{M}}(k)
=
1+|k|^2,
\qquad
\widehat{\mathcal{L}}(k)
=
\varepsilon^2|k|^2+\alpha,
\]
and hence
\[
\widehat{A}(k)
=
1
+
\tau\,
\widehat{\mathcal{M}}(k)
\widehat{\mathcal{L}}(k).
\]
Thus the linear systems can be efficiently solved by FFT.

In the fully discrete energy estimate, let \(M_h\) and \(L_h\) denote the
Fourier spectral discretizations of \(\mathcal{M}=I-\Delta\) and
\(\mathcal{L}=-\varepsilon^2\Delta+\alpha I\), respectively. Their Fourier
symbols are the corresponding multipliers given above, namely
\[
\widehat{M_h}(k)=1+|k|^2,
\qquad
\widehat{L_h}(k)=\varepsilon^2|k|^2+\alpha .
\]
Thus \(M_h\) and \(L_h\) are self-adjoint positive definite with respect to the
discrete inner product, which is still denoted by \(\ip{\cdot}{\cdot}\). The
modified SAV energy is defined by
\[
\Etilde(U^n,r^n)
=
\frac12
\sum_{a=1}^K
\ip{u_a^n}{L_hu_a^n}
+(r^n)^2 .
\]

Then we have the following theorem, which shows that
scheme~\eqref{eq:sav-discrete} is energy stable for the gradient
flow~\eqref{eq:gradient flow}.

\begin{theorem}[Discrete modified-energy identity]\label{thm:stability}
	Let $U^n$ be given.  Construct $c^n,e^n,B^n,q^n$ and $r^n$ as above, and let
	$(U^{n+1},r^{n+1})$ be produced by \eqref{eq:sav-discrete}.
	Then the exact identity
	\begin{align}
		&
		\Etilde(U^{n+1},r^{n+1})
		-
		\Etilde(U^n,r^n)                                      \notag\\
		&\quad
		+
		\frac12
		\sum_{a=1}^K
		\ip{u_a^{n+1}-u_a^n}
		{L_h(u_a^{n+1}-u_a^n)}                              \notag\\
		&\quad
		+
		(r^{n+1}-r^n)^2
		+
		\tau
		\sum_{a=1}^K
		\ip{\mu_a^{n+1}}{M_h\mu_a^{n+1}}
		=0
		\label{eq:energy-identity}
	\end{align}
	holds.  In particular,
	\[
	\Etilde(U^{n+1},r^{n+1})
	\le
	\Etilde(U^n,r^n)
	\qquad\text{for every }\tau>0.
	\]
	Thus the SAV--FFT scheme is unconditionally stable with respect to the
	modified energy $\Etilde$.
	The identity is exact at the fully discrete level.
\end{theorem}

\begin{proof}
	First take the discrete inner product of the first equation in
	\eqref{eq:sav-discrete-a} with
	$\tau\mu_a^{n+1}$.  This gives
	\begin{equation}\label{eq:proof-1}
		\ip{u_a^{n+1}-u_a^n}{\mu_a^{n+1}}
		=
		-\tau\ip{M_h\mu_a^{n+1}}{\mu_a^{n+1}}.
	\end{equation}
	Next, from
	\[
	\mu_a^{n+1}=L_hu_a^{n+1}+r^{n+1}q_a^n,
	\]
	we obtain
	\begin{align}
		\ip{\mu_a^{n+1}}{u_a^{n+1}-u_a^n}
		&=
		\ip{L_hu_a^{n+1}}{u_a^{n+1}-u_a^n}
		+
		r^{n+1}\ip{q_a^n}{u_a^{n+1}-u_a^n}. \label{eq:proof-2}
	\end{align}
	Since $L_h$ is self-adjoint,
	\[
	\ip{L_hx}{x-y}
	=
	\frac12
	\left(
	\ip{x}{L_hx}
	-\ip{y}{L_hy}
	+\ip{x-y}{L_h(x-y)}
	\right).
	\]
	Applying this with $x=u_a^{n+1}$ and $y=u_a^n$ yields
	\begin{align}
		\ip{L_hu_a^{n+1}}{u_a^{n+1}-u_a^n}
		=
		\frac12
		\Bigl(
		\ip{u_a^{n+1}}{L_hu_a^{n+1}}
		-\ip{u_a^n}{L_hu_a^n}
		+\ip{u_a^{n+1}-u_a^n}{L_h(u_a^{n+1}-u_a^n)}
		\Bigr).
		\label{eq:proof-3}
	\end{align}
	
	We now convert the $r$-coupling into a difference of squares.  Multiplying
	the third equation in \eqref{eq:sav-discrete-c} by $2r^{n+1}$ and using
	\[
	2x(x-y)=x^2-y^2+(x-y)^2
	\]
	gives
	\begin{equation}\label{eq:proof-r}
		(r^{n+1})^2-(r^n)^2+(r^{n+1}-r^n)^2
		=
		r^{n+1}
		\sum_{a=1}^K
		\ip{q_a^n}{u_a^{n+1}-u_a^n}.
	\end{equation}
	
	Because the inner product is symmetric, the left side of
	\eqref{eq:proof-1} equals the left side of \eqref{eq:proof-2}.  Summing
	\eqref{eq:proof-1}--\eqref{eq:proof-3} over $a$ and inserting
	\eqref{eq:proof-r}, we get
	\begin{align*}
		&
		-\tau
		\sum_{a=1}^K
		\ip{M_h\mu_a^{n+1}}{\mu_a^{n+1}}
		\\
		&=
		\frac12\sum_{a=1}^K
		\Bigl(
		\ip{u_a^{n+1}}{L_hu_a^{n+1}}
		-\ip{u_a^n}{L_hu_a^n}
		+\ip{u_a^{n+1}-u_a^n}{L_h(u_a^{n+1}-u_a^n)}
		\Bigr)
		\\
		&\quad
		+
		(r^{n+1})^2-(r^n)^2+(r^{n+1}-r^n)^2 .
	\end{align*}
	Moving all terms to the left side gives \eqref{eq:energy-identity}.
	Since $L_h$ and $M_h$ are self-adjoint positive definite, the last three
	terms in \eqref{eq:energy-identity} are nonnegative.  Therefore
	\[
	\Etilde(U^{n+1},r^{n+1})
	\le
	\Etilde(U^n,r^n)
	\]
	for all $\tau>0$.
	
\end{proof}
\subsection{Numerical Verification of Energy Stability}
To illustrate the stability of the SAV scheme \eqref{eq:sav-discrete}, we apply it to real bone images, as shown in Fig. \ref{fig:seg-four-inline}. We conduct a series of experiments using varying time steps, keeping the initial level set position the same and performing 1000 iterations to ensure reliable results. The experiments show that accurate image segmentation outcomes can be achieved with different time steps when using the SAV scheme. With smaller time steps, the energy decay rate during the initial evolution process is relatively slow (Fig. \eqref{fig:mod-energy-dt001}, Fig. \eqref{fig:mod-energy-dt01}). Moreover, as the time steps increase, the number of iterations required for the energy to stabilize significantly decreases (Fig. \eqref{fig:mod-energy-dt1}, Fig. \eqref{fig:mod-energy-dt10000}), indicating potential computational efficiency gains at larger time steps. Furthermore, the modified energy consistently declines over iterations, as demonstrated in Fig. \ref{fig:mod-energy-four-22}, confirming the robustness and effectiveness of the SAV scheme under various conditions.

\begin{figure}[htbp]
	\centering
	\begin{subfigure}[t]{0.21\textwidth}
		\centering
		\includegraphics[
		width=\textwidth,
		height=0.20\textheight,
		keepaspectratio
		]{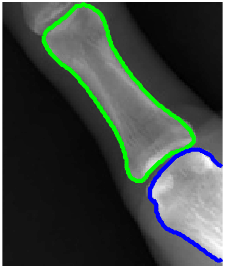}
		\caption{$\tau=0.01$}
		\label{fig:seg-dt001}
	\end{subfigure}
	\hfill
	\begin{subfigure}[t]{0.21\textwidth}
		\centering
		\includegraphics[
		width=\textwidth,
		height=0.20\textheight,
		keepaspectratio
		]{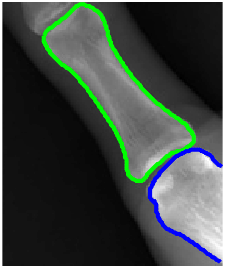}
		\caption{$\tau=0.1$}
		\label{fig:seg-dt01}
	\end{subfigure}
	\hfill
	\begin{subfigure}[t]{0.21\textwidth}
		\centering
		\includegraphics[
		width=\textwidth,
		height=0.20\textheight,
		keepaspectratio
		]{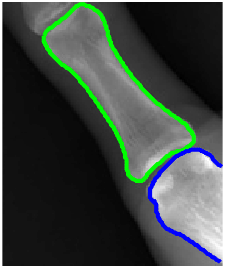}
		\caption{$\tau=1$}
		\label{fig:seg-dt1}
	\end{subfigure}
	\hfill
	\begin{subfigure}[t]{0.21\textwidth}
		\centering
		\includegraphics[
		width=\textwidth,
		height=0.20\textheight,
		keepaspectratio
		]{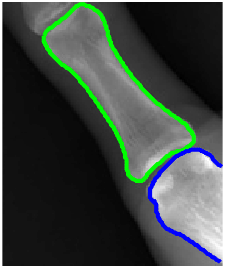}
		\caption{$\tau=10000$}
		\label{fig:seg-dt10000}
	\end{subfigure}
	\caption{Segmentation results obtained with different time steps.}
	\label{fig:seg-four-inline}
\end{figure}

\begin{figure}[htbp]
	\centering
	\begin{subfigure}[t]{0.42\textwidth}
		\centering
		\includegraphics[
		width=\textwidth,
		height=0.22\textheight,
		keepaspectratio
		]{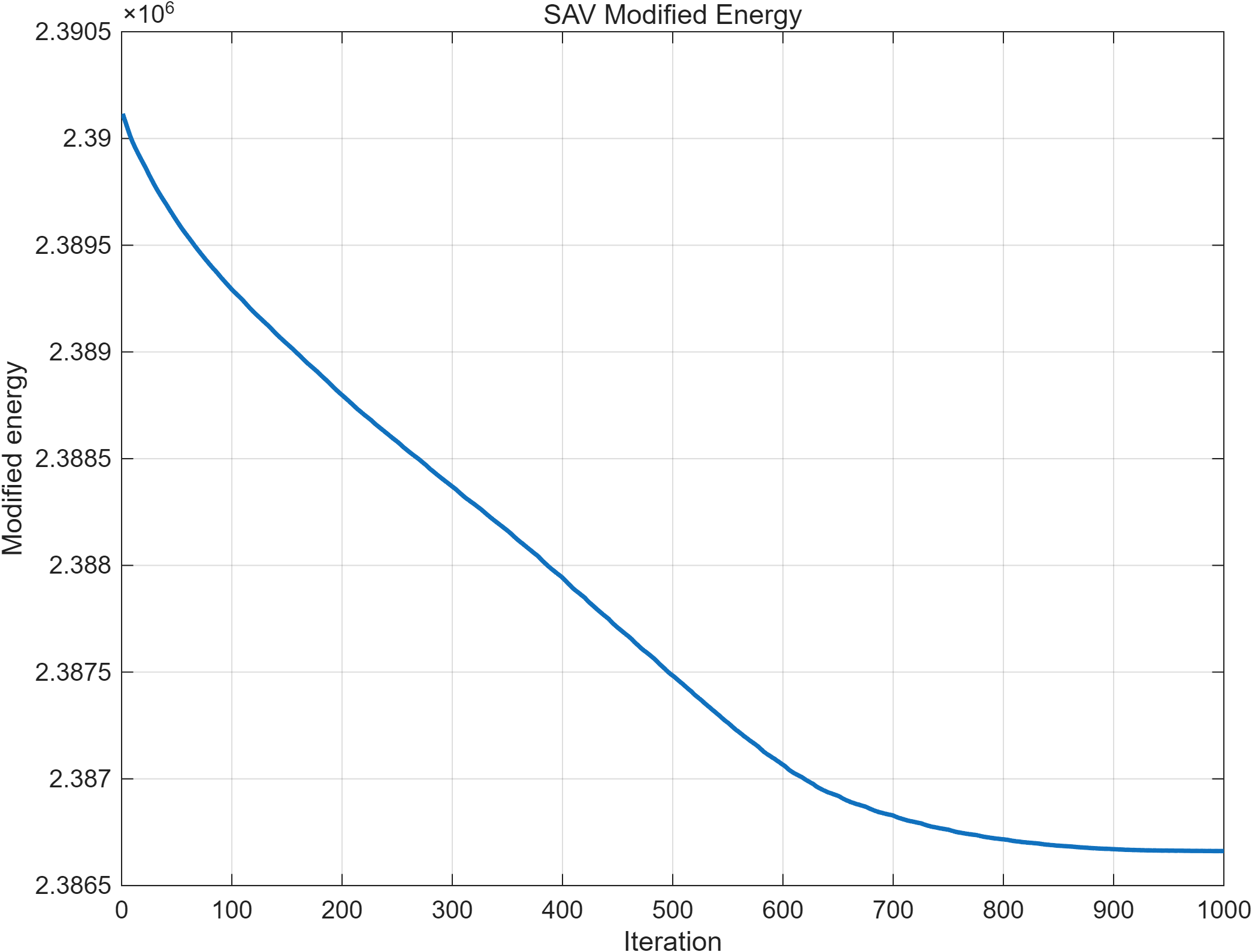}
		\caption{$\tau=0.01$}
		\label{fig:mod-energy-dt001}
	\end{subfigure}
	\hspace{0.04\textwidth}
	\begin{subfigure}[t]{0.42\textwidth}
		\centering
		\includegraphics[
		width=\textwidth,
		height=0.22\textheight,
		keepaspectratio
		]{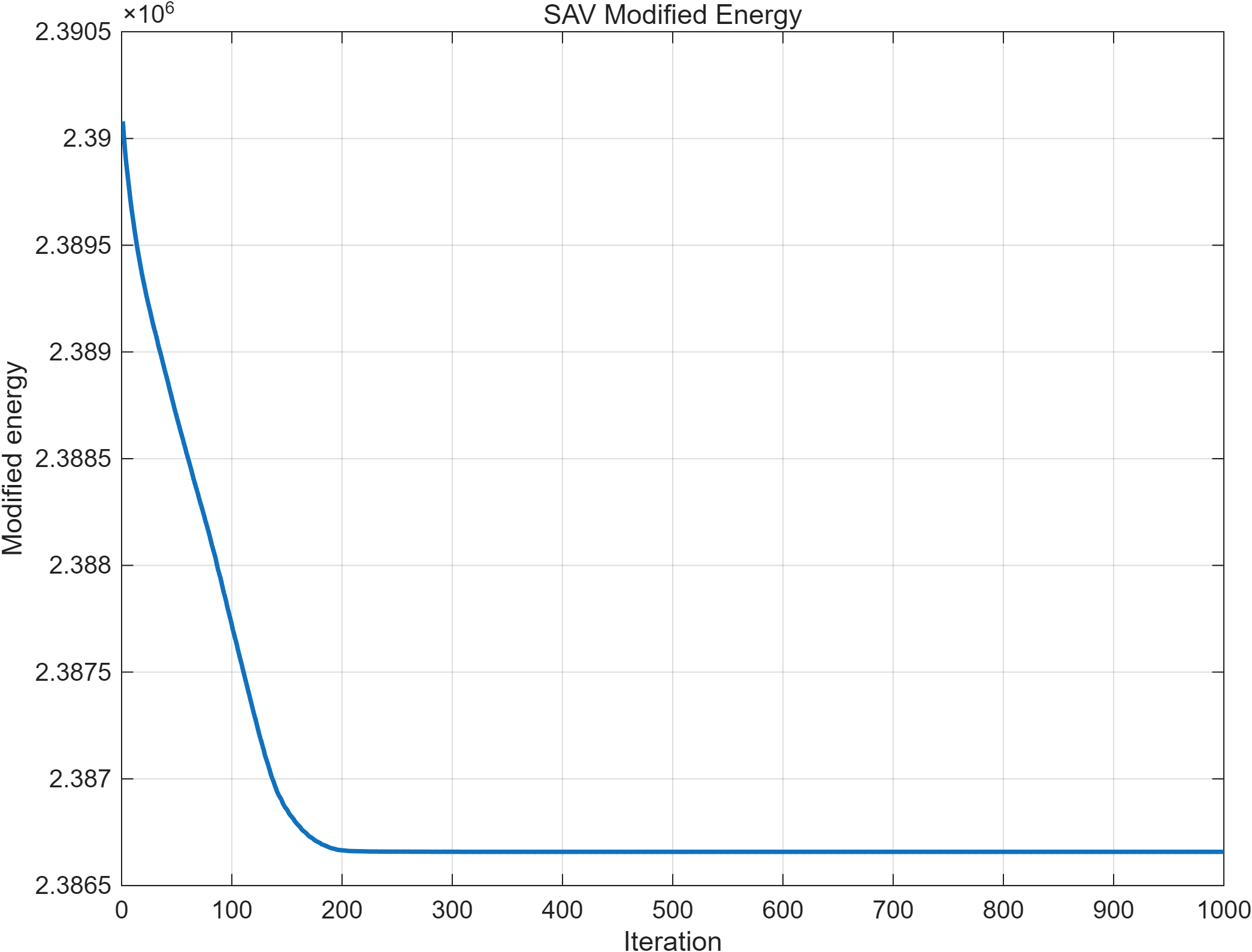}
		\caption{$\tau=0.1$}
		\label{fig:mod-energy-dt01}
	\end{subfigure}
	
	\vspace{0.12cm}
	
	\begin{subfigure}[t]{0.42\textwidth}
		\centering
		\includegraphics[
		width=\textwidth,
		height=0.22\textheight,
		keepaspectratio
		]{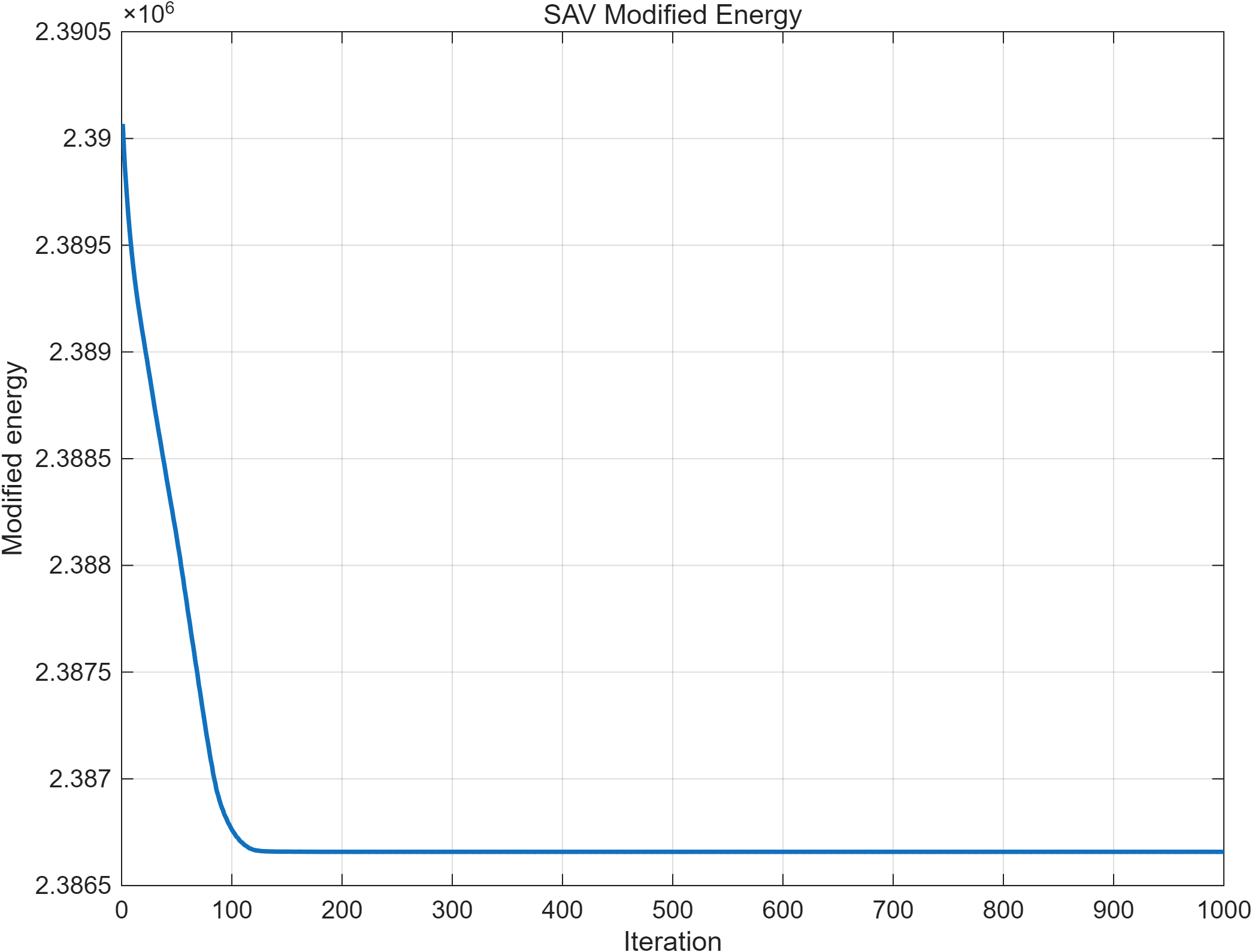}
		\caption{$\tau=1$}
		\label{fig:mod-energy-dt1}
	\end{subfigure}
	\hspace{0.04\textwidth}
	\begin{subfigure}[t]{0.42\textwidth}
		\centering
		\includegraphics[
		width=\textwidth,
		height=0.22\textheight,
		keepaspectratio
		]{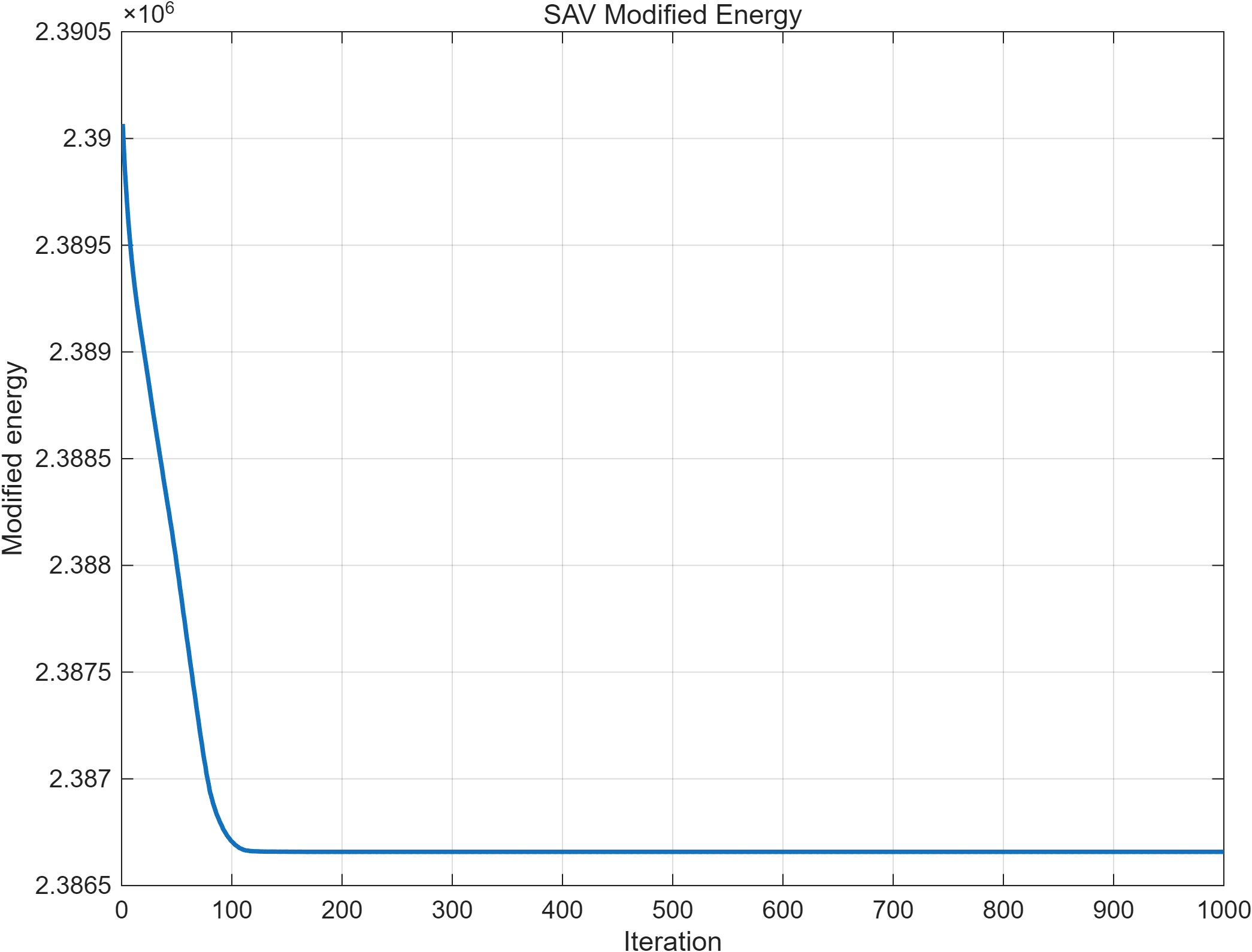}
		\caption{$\tau=10000$}
		\label{fig:mod-energy-dt10000}
	\end{subfigure}
	\caption{Modified energy curves obtained with different time steps.}
	\label{fig:mod-energy-four-22}
\end{figure}

This stabilized SAV splitting strategy addresses the shortcomings of the two methods. The stabilized splitting plays a particularly important role in the large-time-step implementation. By adding and subtracting the stabilization term $\alpha u_a$, the stiff part
\[
-\varepsilon^2\Delta u_a+\alpha u_a
\]
is incorporated into the implicit linear operator
\[
\mathcal L=-\varepsilon^2\Delta+\alpha I .
\]
Consequently, the main operator
\[
A=I+\tau\,\mathcal M\mathcal L
\]
remains positive and well-conditioned in the Fourier implementation when a suitable stabilization parameter $\alpha$ is used. This treatment significantly relaxes the time-step restriction caused by the fourth-order Cahn--Hilliard dynamics. In particular, under appropriate choices of the stabilization parameter and the SAV shift constant $C_0$, the proposed scheme can remain stable even for very large time steps, for example $\tau=10^4$ in our numerical tests.

Meanwhile, the complicated nonlinear potential and the softmax-based data term are absorbed into the SAV nonlinear part and treated explicitly through the scalar auxiliary variable. Therefore, the method avoids solving nonlinear systems at each time step, while still preserving a modified discrete energy dissipation law. As a result, the proposed scheme is linear, robust for large time steps, and computationally efficient.


\section{Experimental results}
\label{sec:experiments}

In this section, we present a series of numerical experiments to evaluate the effectiveness and robustness of the proposed model. 
We first compare the proposed method with three representative variational baselines:
the CV model~\cite{chanActiveContoursEdges2001}, the
RSF model~\cite{chunmingliMinimizationRegionScalableFitting2008a},
the CP-ICTM model~\cite{dengConnectedComponentPreservingImage2025}. 
These comparisons are used to demonstrate the advantage of the proposed model in low-contrast and weak-boundary segmentation, especially for adjacent homogeneous structures with similar intensity distributions. 
To further investigate the role of the Cahn--Hilliard regularization in interface evolution, we also compare the proposed method with another fourth-order phase-field regularization method, namely the MBE regularization. 
In addition, experiments on the SKI10 dataset~\cite{heimannSegmentationKneeImages2010} are conducted to compare the proposed method with mainstream deep learning models, including U-Net~\cite{ronnebergerUNetConvolutionalNetworks2015a}, TransUNet~\cite{chenTransUNetTransformersMake2021}, Swin-UNet~\cite{caoSwinUnetUnetlikePure2021}, and MedSegDiff~\cite{wuMedSegDiffMedicalImage}. 
The results show that the proposed model achieves competitive segmentation accuracy and exhibits particular advantages in boundary stability, especially as measured by boundary-sensitive metrics such as HD95.

\subsection{The Boundary Accuracy and Stability of the Proposed Model}

We first evaluate the performance of the proposed model in segmenting low-contrast boundaries. The obtained results show that the boundaries are clearly delineated and well separated, demonstrating the accuracy and robustness of the proposed method in boundary segmentation.

As shown in Figure \ref{fig:5p1}, we evaluate the ability of the proposed model to segment a low-contrast boundary image. The original image in Figure~\eqref{fig:7p2} contains adjacent anatomical structures with weak intensity differences, where the boundary between neighboring regions is not clearly separated by image contrast alone. This makes the segmentation task challenging, especially near the contact area between the two structures. The segmentation result in Figure~\eqref{fig:7p2-segment} shows that the proposed method can accurately capture the target boundaries and separate the neighboring regions without undesirable merging. The extracted contours are coherent and well aligned with the visible object interfaces, indicating the robustness of the proposed model in low-contrast regions.

The corresponding phase-field functions are further displayed in Figures~\eqref{fig:u_12} and~\eqref{fig:u_3}. Figure~\eqref{fig:u_12} shows the three-dimensional surfaces of $u_1$ and $u_2$, while Figure~\eqref{fig:u_3} presents the surface of $u_3$. These phase-field variables form clear plateau regions inside different segmented phases and exhibit smooth transitions near the interfaces. This behavior illustrates the effect of the Cahn--Hilliard regularization in promoting phase separation and stable interfacial evolution. Meanwhile, the softmax-based multiphase representation allows different regions to compete with each other in a continuous manner. These results demonstrate that the proposed model can produce accurate and stable segmentation results for weak-boundary and low-contrast images.

\begin{figure}[htbp]
	\centering
	\begin{subfigure}[t]{0.42\textwidth}
		\centering
		\includegraphics[
		width=\textwidth,
		height=0.22\textheight,
		keepaspectratio
		]{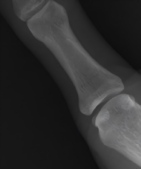}
		\caption{The initial image}
		\label{fig:7p2}
	\end{subfigure}
	\hspace{0.04\textwidth}
	\begin{subfigure}[t]{0.42\textwidth}
		\centering
		\includegraphics[
		width=\textwidth,
		height=0.22\textheight,
		keepaspectratio
		]{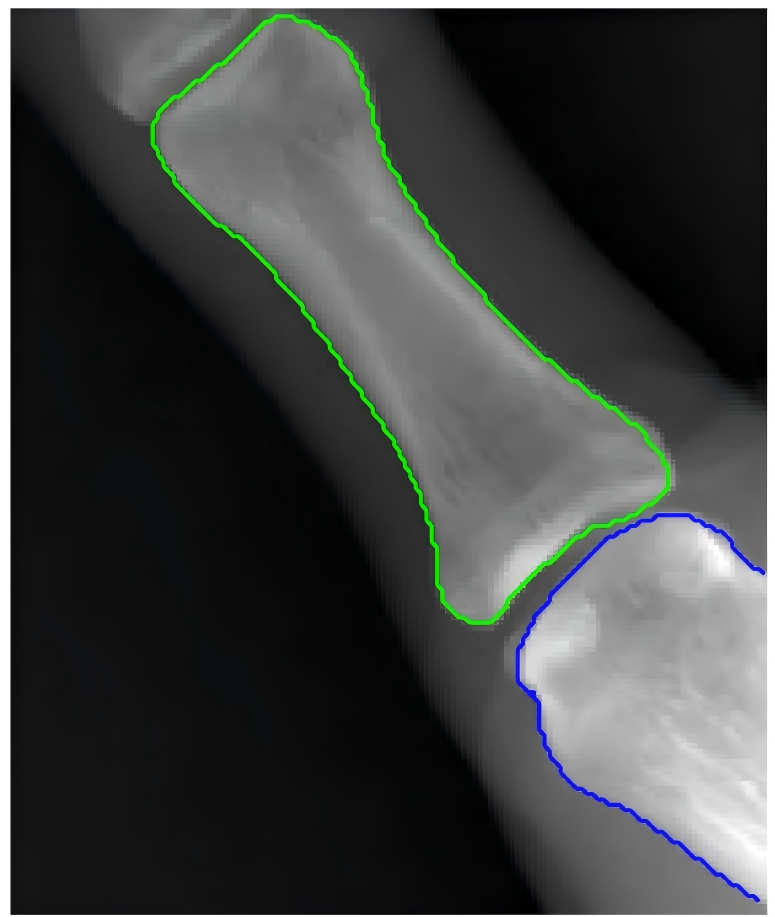}
		\caption{Segmentation result}
		\label{fig:7p2-segment}
	\end{subfigure}
	
	\vspace{0.12cm}
	
	\begin{subfigure}[t]{0.42\textwidth}
		\centering
		\includegraphics[
		width=\textwidth,
		height=0.22\textheight,
		keepaspectratio
		]{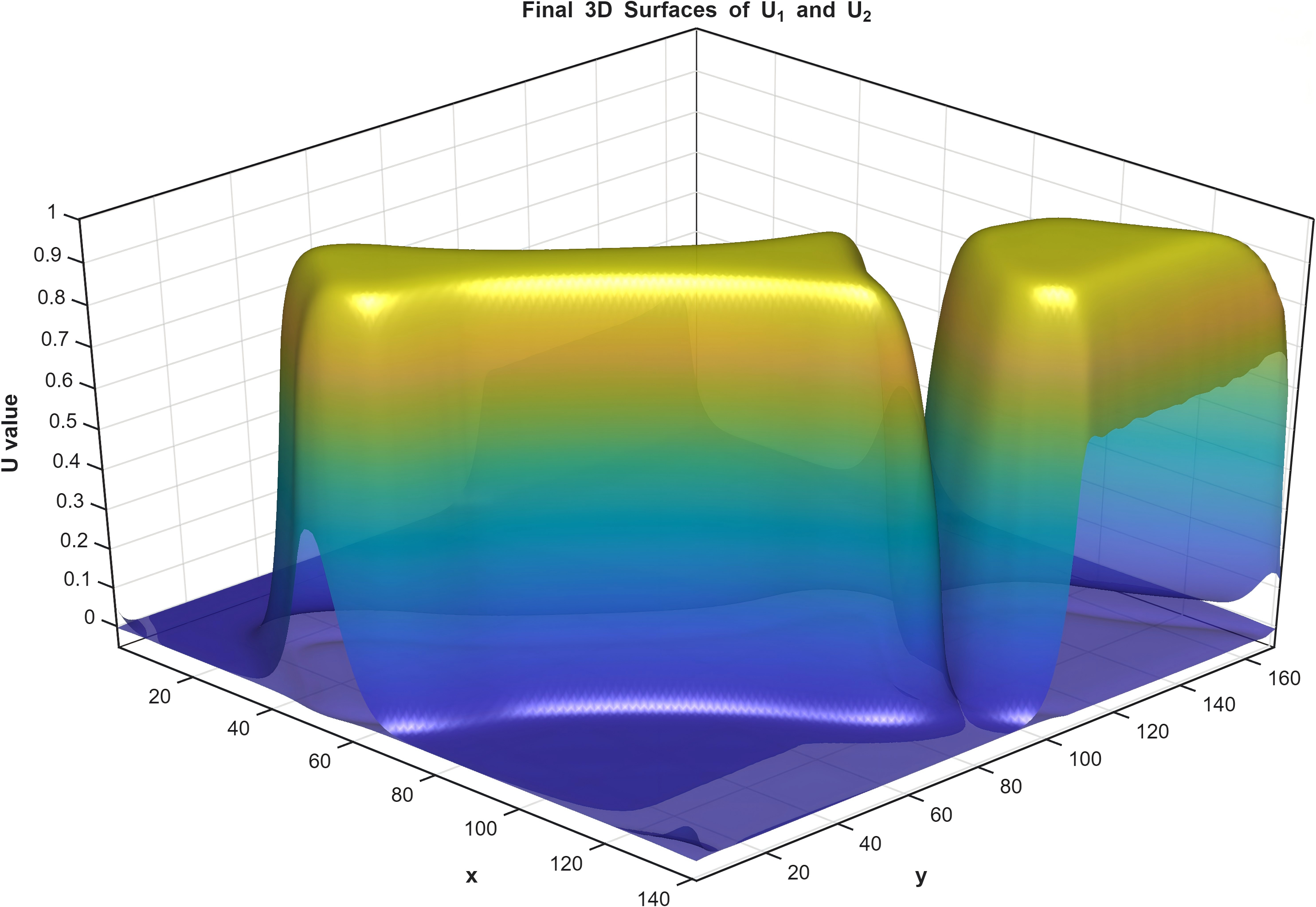}
		\caption{$u_1,u_2$}
		\label{fig:u_12}
	\end{subfigure}
	\hspace{0.04\textwidth}
	\begin{subfigure}[t]{0.42\textwidth}
		\centering
		\includegraphics[
		width=\textwidth,
		height=0.22\textheight,
		keepaspectratio
		]{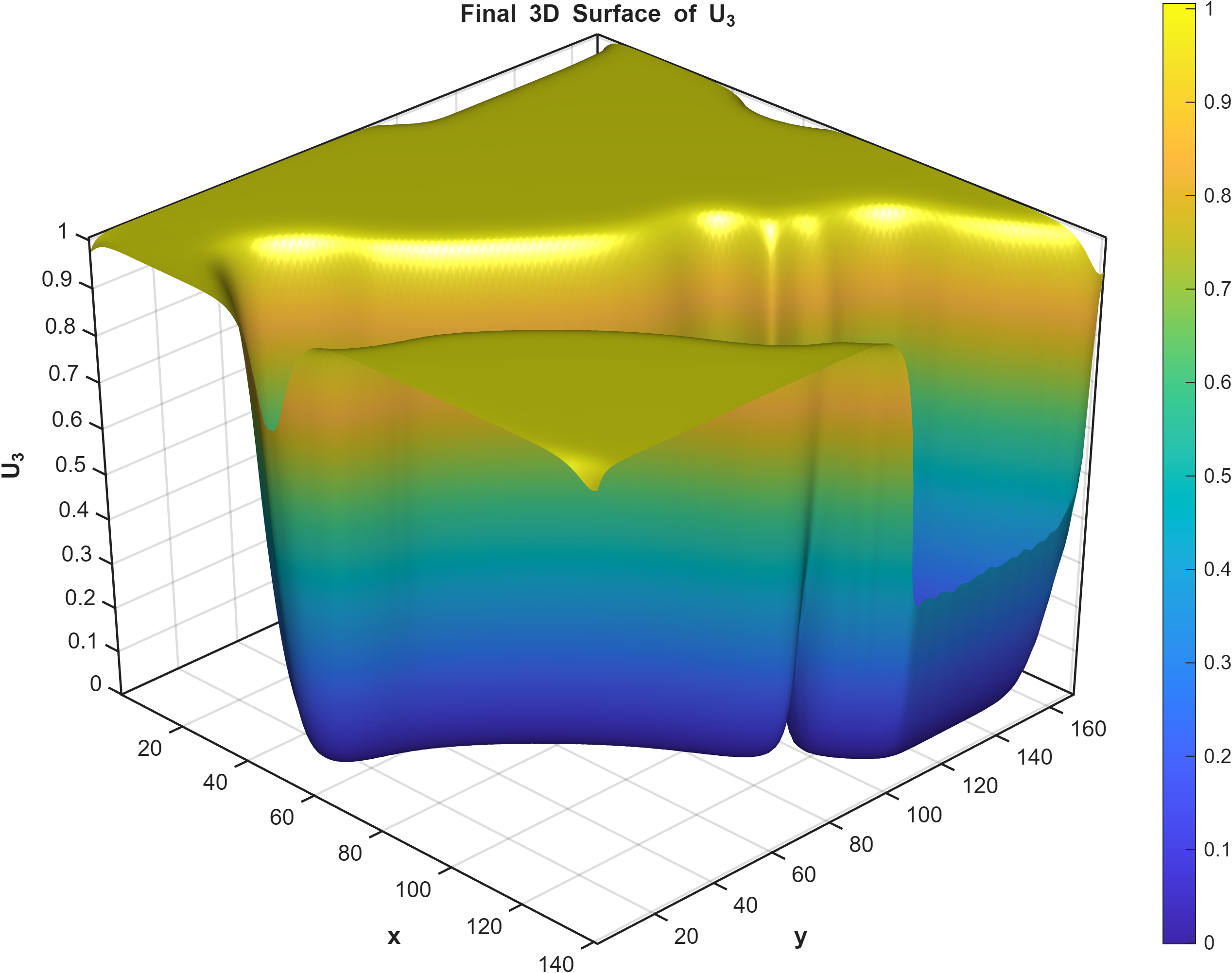}
		\caption{$u_3$}
		\label{fig:u_3}
	\end{subfigure}
	\caption{Segmentation of a low-contrast boundary image and the corresponding phase-field representations.}
	\label{fig:5p1}
\end{figure}

In summary, the proposed model achieves stable and accurate segmentation in low-contrast regions by clearly separating adjacent structures and preserving weak boundary details, demonstrating its effectiveness in both boundary sensitivity and interfacial stability.

\subsection{Ablation and parameter stability analysis}

We next examine the contribution of the Cahn--Hilliard regularization and the
sensitivity of the proposed model to the parameters \(\varepsilon\) and
\(\beta\). All experiments in this subsection use the same image and
initialization. When one parameter is varied, the remaining parameters are kept
fixed. The stopping tolerance for the relative change of the phase-field
variables is \(10^{-5}\), and the maximum number of iterations is \(1500\).

Figure~\ref{fig:regularization-ablation} compares the complete model with its
unregularized counterpart, in which the Cahn--Hilliard energy is removed and
only the softmax-based data fitting term is retained. Without regularization,
the data term identifies the main regions but produces a less coherent
partition near the low-contrast contact area. In particular, the contours are
locally displaced and the narrow separation between the adjacent structures is
not recovered reliably. By contrast, the complete model preserves smooth
boundaries while maintaining a clear separation near the weak interface. This
comparison indicates that the Cahn--Hilliard term contributes more than contour
smoothing: it stabilizes the phase transition when the local data force is
insufficient to determine the nearby interfaces.

\begin{figure}[htbp]
	\centering
	\begin{subfigure}[t]{0.42\textwidth}
		\centering
		\includegraphics[
		width=\textwidth,
		height=0.25\textheight,
		keepaspectratio
		]{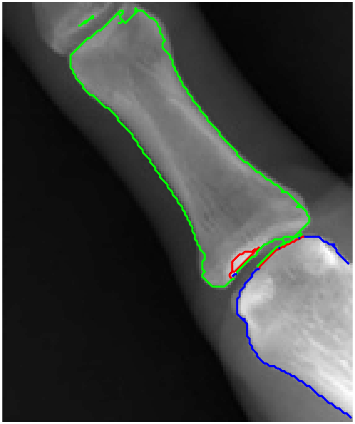}
		\caption{Without Cahn--Hilliard regularization}
		\label{fig:ablation-without-regularization}
	\end{subfigure}
	\hspace{0.05\textwidth}
	\begin{subfigure}[t]{0.42\textwidth}
		\centering
		\includegraphics[
		width=\textwidth,
		height=0.25\textheight,
		keepaspectratio
		]{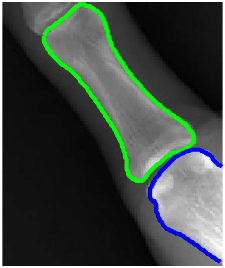}
		\caption{Complete model}
		\label{fig:ablation-complete-model}
	\end{subfigure}
	\caption{Ablation study of the Cahn--Hilliard regularization on a
	low-contrast image.}
	\label{fig:regularization-ablation}
\end{figure}

The parameter \(\varepsilon\) controls the contribution of the gradient term
and the characteristic scale of the diffuse interface. As shown in
Figure~\ref{fig:epsilon-stability}, the principal structures and their narrow
separation are consistently recovered for \(\varepsilon=0.5,1,\) and \(2\).
Table~\ref{tab:epsilon-stability} further shows that all three cases satisfy the
prescribed stopping tolerance, with \(\varepsilon=2\) requiring the fewest
iterations. For \(\varepsilon=4\) and \(8\), the principal segmentation remains
recognizable, but the iterations reach the prescribed upper limit without
satisfying the stopping criterion. The corresponding relative changes remain
of order \(10^{-2}\), indicating that an excessively large \(\varepsilon\)
can make the evolution more difficult to resolve and slow convergence under
the fixed numerical settings.

\begin{figure}[htbp]
	\centering
	\begin{subfigure}[t]{0.18\textwidth}
		\centering
		\includegraphics[width=\textwidth]{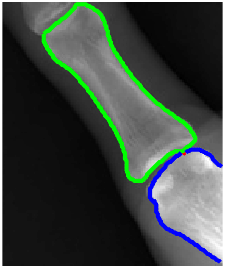}
		\caption{\(\varepsilon=0.5\)}
	\end{subfigure}
	\hfill
	\begin{subfigure}[t]{0.18\textwidth}
		\centering
		\includegraphics[width=\textwidth]{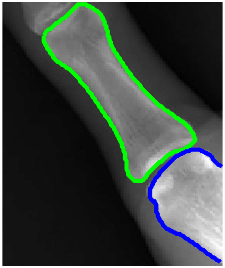}
		\caption{\(\varepsilon=1\)}
	\end{subfigure}
	\hfill
	\begin{subfigure}[t]{0.18\textwidth}
		\centering
		\includegraphics[width=\textwidth]{pictures/epsilon_test_results/segmentation_epsilon_2.png}
		\caption{\(\varepsilon=2\)}
	\end{subfigure}
	\hfill
	\begin{subfigure}[t]{0.18\textwidth}
		\centering
		\includegraphics[width=\textwidth]{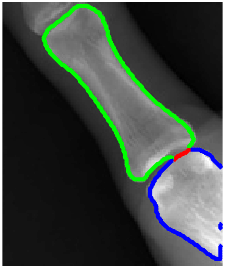}
		\caption{\(\varepsilon=4\)}
	\end{subfigure}
	\hfill
	\begin{subfigure}[t]{0.18\textwidth}
		\centering
		\includegraphics[width=\textwidth]{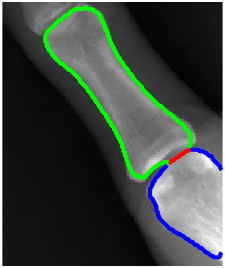}
		\caption{\(\varepsilon=8\)}
	\end{subfigure}
	\caption{Segmentation results for different values of \(\varepsilon\).}
	\label{fig:epsilon-stability}
\end{figure}

\begin{table}[htbp]
	\centering
	\caption{Convergence results for different values of \(\varepsilon\).}
	\label{tab:epsilon-stability}
	\begin{tabular}{ccc}
		\hline
		\(\varepsilon\) & Iterations & Final relative change \\
		\hline
		\(0.5\) & \(1296\) & \(9.9381\times10^{-6}\) \\
		\(1\)   & \(576\)  & \(9.9218\times10^{-6}\) \\
		\(2\)   & \(352\)  & \(9.9196\times10^{-6}\) \\
		\(4\)   & \(1500\) & \(1.5660\times10^{-2}\) \\
		\(8\)   & \(1500\) & \(1.5428\times10^{-2}\) \\
		\hline
	\end{tabular}
\end{table}

The softmax parameter \(\beta\) controls the sharpness of the phase assignment.
Figure~\ref{fig:beta-stability} shows that the segmentation boundaries are
visually consistent for \(\beta=2,5,\) and \(10\), demonstrating that the model
is not sensitive to moderate changes in the softmax sharpness. According to
Table~\ref{tab:beta-stability}, these three choices also satisfy the stopping
tolerance, and \(\beta=5\) gives the fastest convergence in this experiment.
When \(\beta\) is increased to \(20\) or \(40\), the softmax assignment becomes
considerably sharper. Although the main segmentation geometry is retained, the
iterations do not reach the stopping tolerance within \(1500\) steps. Thus, a
very large \(\beta\) can amplify the nonlinearity of the data term and reduce
the numerical efficiency of the evolution under the present discretization.

\begin{figure}[htbp]
	\centering
	\begin{subfigure}[t]{0.18\textwidth}
		\centering
		\includegraphics[width=\textwidth]{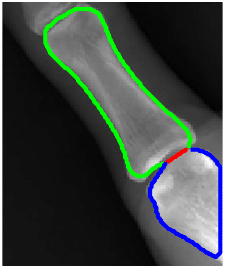}
		\caption{\(\beta=2\)}
	\end{subfigure}
	\hfill
	\begin{subfigure}[t]{0.18\textwidth}
		\centering
		\includegraphics[width=\textwidth]{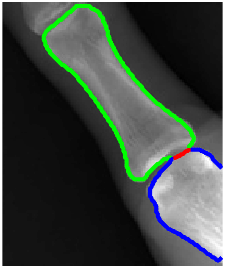}
		\caption{\(\beta=5\)}
	\end{subfigure}
	\hfill
	\begin{subfigure}[t]{0.18\textwidth}
		\centering
		\includegraphics[width=\textwidth]{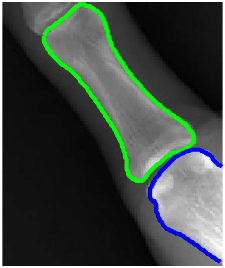}
		\caption{\(\beta=10\)}
	\end{subfigure}
	\hfill
	\begin{subfigure}[t]{0.18\textwidth}
		\centering
		\includegraphics[width=\textwidth]{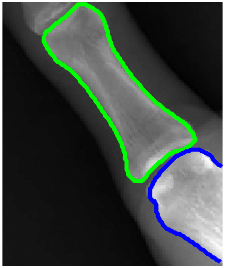}
		\caption{\(\beta=20\)}
	\end{subfigure}
	\hfill
	\begin{subfigure}[t]{0.18\textwidth}
		\centering
		\includegraphics[width=\textwidth]{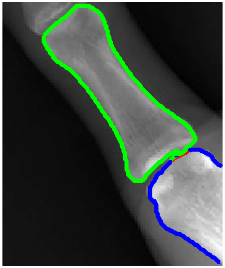}
		\caption{\(\beta=40\)}
	\end{subfigure}
	\caption{Segmentation results for different values of \(\beta\).}
	\label{fig:beta-stability}
\end{figure}

\begin{table}[htbp]
	\centering
	\caption{Convergence results for different values of \(\beta\).}
	\label{tab:beta-stability}
	\begin{tabular}{ccc}
		\hline
		\(\beta\) & Iterations & Final relative change \\
		\hline
		\(2\)  & \(426\)  & \(9.9155\times10^{-6}\) \\
		\(5\)  & \(275\)  & \(9.9738\times10^{-6}\) \\
		\(10\) & \(352\)  & \(9.9196\times10^{-6}\) \\
		\(20\) & \(1500\) & \(3.3391\times10^{-3}\) \\
		\(40\) & \(1500\) & \(1.1037\times10^{-2}\) \\
		\hline
	\end{tabular}
\end{table}

Overall, the ablation result confirms the importance of the Cahn--Hilliard
regularization for preserving weak separation boundaries. The parameter tests
also show that the model produces consistent segmentation results over
moderate ranges of \(\varepsilon\) and \(\beta\), while excessively large
values can substantially slow the numerical convergence.

\subsection{Robustness to noise and low contrast}

Accurate segmentation in noisy images remains challenging, since small-scale perturbations may obscure weak boundary information and disturb the evolution process. To evaluate the robustness of the proposed model, we conduct experiments on low-contrast images corrupted by different types of noise, including additive noise and multiplicative noise.

As shown in Figures~\ref{fig:ch-noisy-7p2} and~\ref{fig:ch-noisy-robustness-10095}, the proposed model is tested under two noisy low-contrast settings. Compared with the clean images in Figures~\eqref{fig:ch-7P2} and~\eqref{fig:ch-noisy-clean-10095}, the image in Figure~\eqref{fig:ch-noisy-7P2} is corrupted by additive noise with intensity $0.8$, while the image in Figure~\eqref{fig:ch-noisy-input-10095} is corrupted by multiplicative noise with intensity $0.4$. These noise perturbations significantly degrade the image quality, weaken local intensity contrast, and obscure parts of the object boundaries.

Nevertheless, the proposed model still produces stable segmentation results. For the image with additive noise, Figure~\eqref{fig:ch-noisy-7P2-SEG} shows that the target bone structure is clearly extracted from the noisy background. The adjacent regions remain well separated, and the contour follows the main boundary without obvious leakage or fragmentation. This indicates that the proposed model is robust to strong additive intensity perturbations.

For the larger image with multiplicative noise, the segmentation task becomes more challenging because of the increased image size, multiple adjacent bone structures, and spatially varying noise interference. Even in this case, Figure~\eqref{fig:ch-noisy-seg-10095} shows that the proposed model can still separate different bone regions clearly. The obtained contours remain coherent and visually sharp. These results demonstrate that the proposed Cahn--Hilliard regularization effectively suppresses noise-driven oscillations while preserving stable and clear boundary separation in low-contrast images.

\begin{figure}[htbp]
	\centering
	\begin{subfigure}[t]{0.28\textwidth}
		\centering
		\includegraphics[
		width=\textwidth,
		height=0.28\textheight,
		keepaspectratio
		]{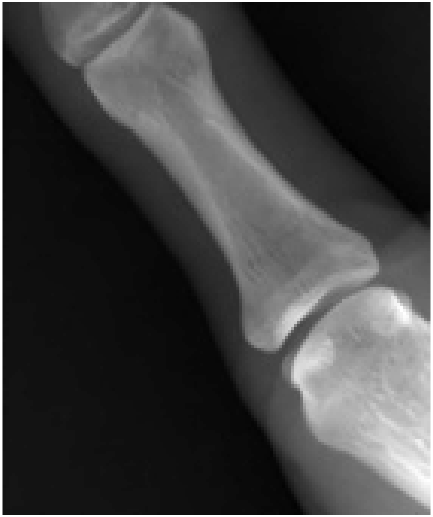}
		\caption{Clean image }
		\label{fig:ch-7P2}
	\end{subfigure}
	\hfill
	\begin{subfigure}[t]{0.28\textwidth}
		\centering
		\includegraphics[
		width=\textwidth,
		height=0.28\textheight,
		keepaspectratio
		]{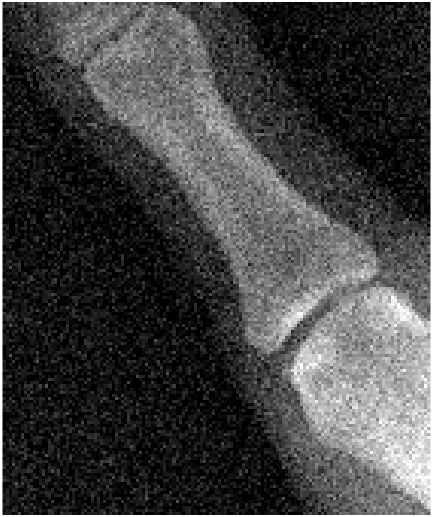}
		\caption{ Noisy image}
		\label{fig:ch-noisy-7P2}
	\end{subfigure}
	\hfill
	\begin{subfigure}[t]{0.28\textwidth}
		\centering
		\includegraphics[
		width=\textwidth,
		height=0.28\textheight,
		keepaspectratio
		]{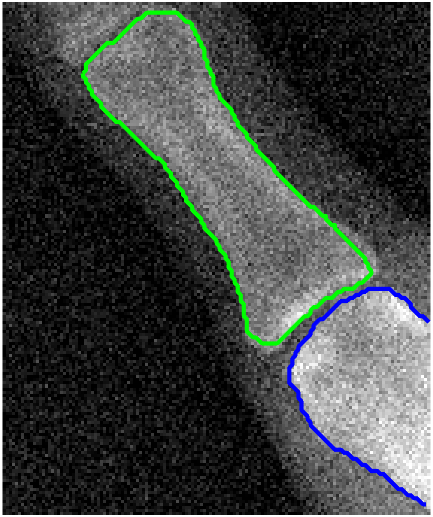}
		\caption{Segmentation result}
		\label{fig:ch-noisy-7P2-SEG}
	\end{subfigure}
	\caption{Segmentation result of the proposed  model on a noisy low-contrast image.}
	\label{fig:ch-noisy-7p2}
\end{figure}

\begin{figure}[htbp]
	\centering
	\begin{subfigure}[t]{0.28\textwidth}
		\centering
		\includegraphics[
		width=\textwidth,
		height=0.28\textheight,
		keepaspectratio
		]{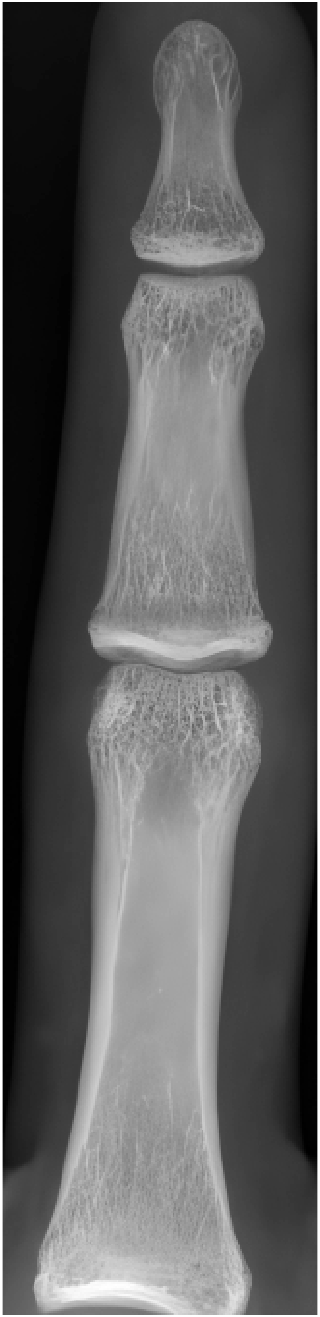}
		\caption{Clean image}
		\label{fig:ch-noisy-clean-10095}
	\end{subfigure}
	\hfill
	\begin{subfigure}[t]{0.28\textwidth}
		\centering
		\includegraphics[
		width=\textwidth,
		height=0.28\textheight,
		keepaspectratio
		]{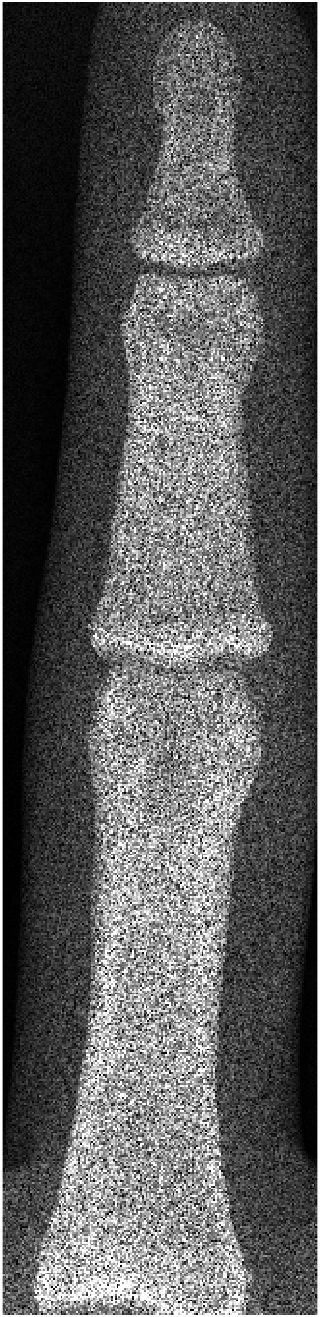}
		\caption{Noisy image}
		\label{fig:ch-noisy-input-10095}
	\end{subfigure}
	\hfill
	\begin{subfigure}[t]{0.28\textwidth}
		\centering
		\includegraphics[
		width=\textwidth,
		height=0.28\textheight,
		keepaspectratio
		]{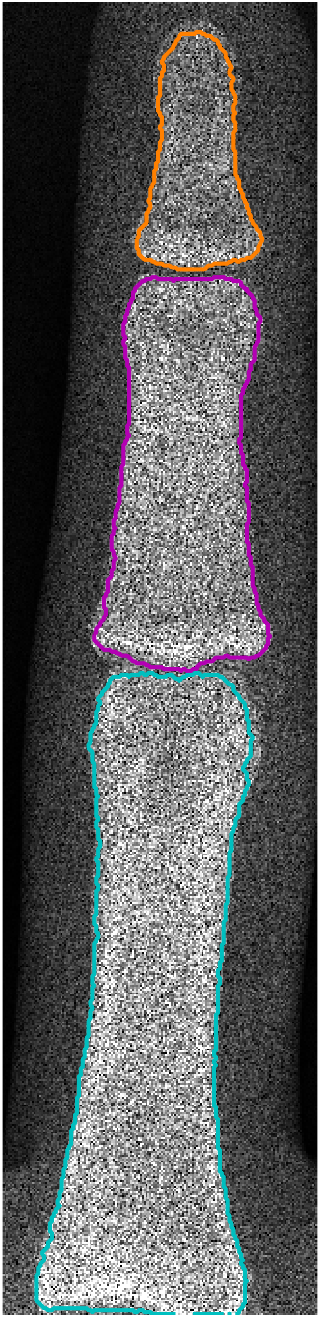}
		\caption{Segmentation result}
		\label{fig:ch-noisy-seg-10095}
	\end{subfigure}
	\caption{Segmentation result of the proposed Cahn--Hilliard model on a large-scale noisy low-contrast image.}
	\label{fig:ch-noisy-robustness-10095}
\end{figure}

\subsection{Comparison with variational and phase-field models}
Figures~\ref{fig:different-models-comparison} and~\ref{fig:different-models-comparison-2} show the segmentation results obtained by different variational and phase-field models on two low-contrast bone images used in this experiment. 
Each figure contains the original image and the segmentation results produced by the proposed method, CV~\cite{chanActiveContoursEdges2001}, RSF~\cite{chunmingliMinimizationRegionScalableFitting2008a}, CP-ICTM~\cite{dengConnectedComponentPreservingImage2025}, and MBE. 
These examples are challenging because the bone boundaries are partially weak, and the adjacent anatomical structures have similar intensity distributions. 
Therefore, an effective method should not only locate the outer bone boundary accurately, but also preserve the separation between neighboring bone regions.

As shown in Figure~\ref{fig:different-models-comparison}, the proposed method accurately extracts the target bone structure and produces a continuous boundary along the weak edge region. 
The contour remains close to the actual anatomical boundary, and the neighboring bone regions are clearly separated. 
The CV model can roughly identify the main object region, but the obtained contour is less accurate near the weak boundary and tends to deviate from the true bone edge. 
Moreover, due to the weak intensity contrast between adjacent bone structures, the CV model fails to separate the neighboring bones clearly and tends to merge them into a single region. 
The RSF model improves the local fitting ability to some extent, but boundary leakage still appears near the adjacent joint region. 
Although the local intensity information helps refine part of the contour, the RSF model still cannot fully distinguish the narrow gap between neighboring bone structures, leading to incomplete separation. 
Although CP-ICTM preserves connected components, its result is strongly affected by the prescribed topological constraint and produces an over-extended contour in the lower right region. 
The MBE-based model gives a relatively smooth result, but its boundary performance is slightly inferior to that of the proposed model. 
In contrast, the proposed model achieves a better balance between boundary smoothness and structural separation.

Figure~\ref{fig:different-models-comparison-2} further demonstrates the advantage of the proposed method on a more elongated bone structure. 
The original image contains multiple bone segments with weak transitions between neighboring regions. 
The proposed method successfully separates the upper, middle, and lower bone parts, and the extracted contours are smooth and coherent. 
In comparison, the CV and RSF models tend to merge adjacent regions or produce inaccurate contours around the joints, since their data fitting terms are insufficient to distinguish regions with similar intensities. 
CP-ICTM can maintain the connectivity of the segmented region, but the resulting boundary is relatively rigid and may fail to accurately follow local weak edges. 
The MBE model also provides a smooth segmentation, but it shows slight boundary inconsistency near the joint regions. 
The proposed method preserves the narrow gaps between adjacent structures more clearly, which is important for accurate bone segmentation.

\begin{figure}[htbp]
	\centering
	
	\begin{subfigure}[t]{0.155\textwidth}
		\centering
		\includegraphics[
		width=\textwidth,
		height=0.13\textheight,
		keepaspectratio
		]{fig/7.2.jpg}
		\caption{Original image}
		\label{fig:compare-original}
	\end{subfigure}
	\hfill
	\begin{subfigure}[t]{0.155\textwidth}
		\centering
		\includegraphics[
		width=\textwidth,
		height=0.13\textheight,
		keepaspectratio
		]{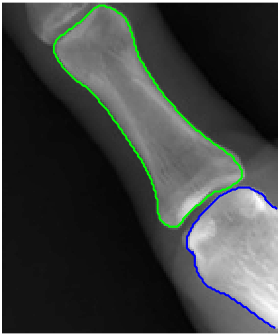}
		\caption{Ours}
		\label{fig:compare-ours}
	\end{subfigure}
	\hfill
	\begin{subfigure}[t]{0.155\textwidth}
		\centering
		\includegraphics[
		width=\textwidth,
		height=0.13\textheight,
		keepaspectratio
		]{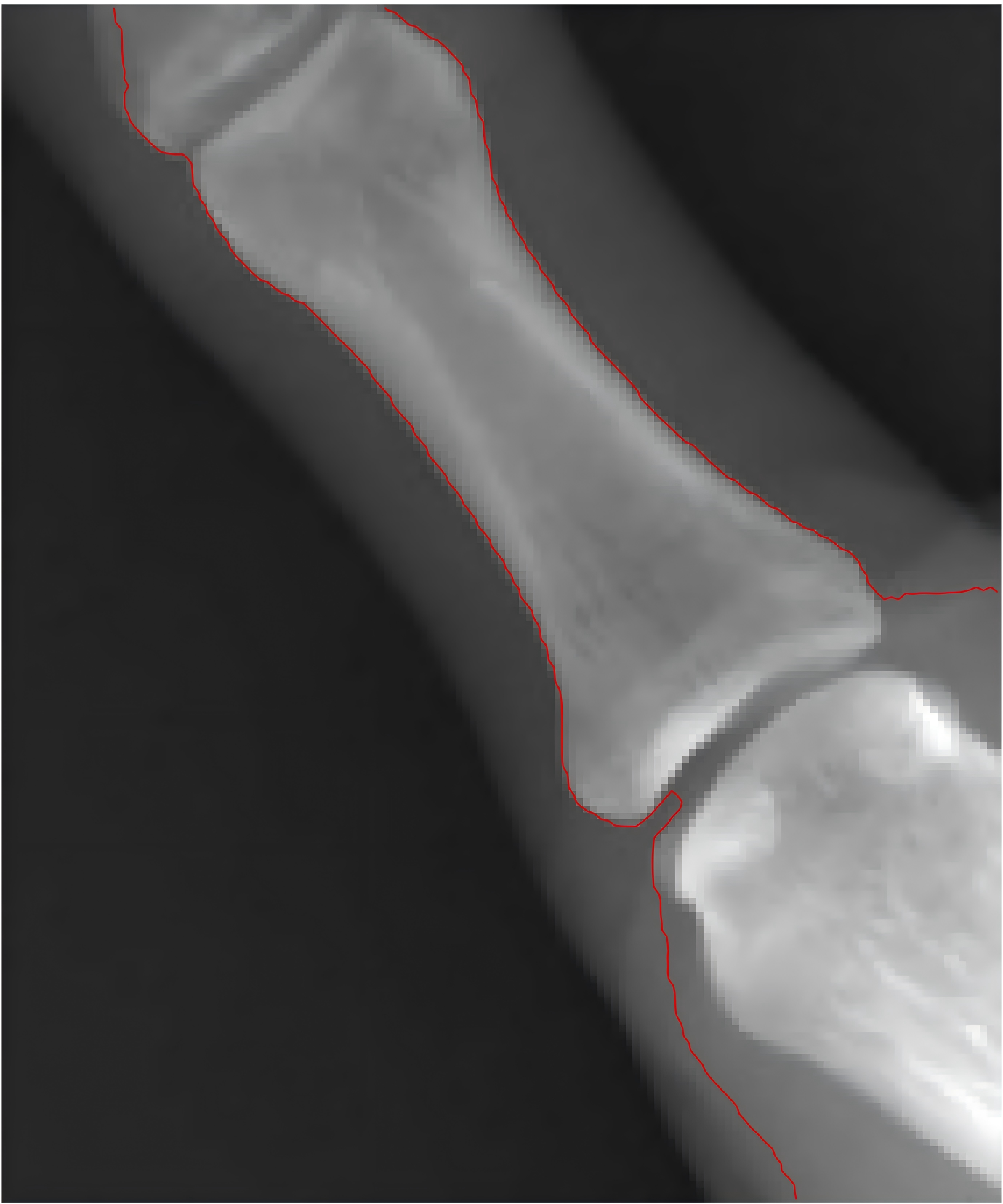}
		\caption{CV}
		\label{fig:compare-cv}
	\end{subfigure}
	\hfill
	\begin{subfigure}[t]{0.155\textwidth}
		\centering
		\includegraphics[
		width=\textwidth,
		height=0.13\textheight,
		keepaspectratio
		]{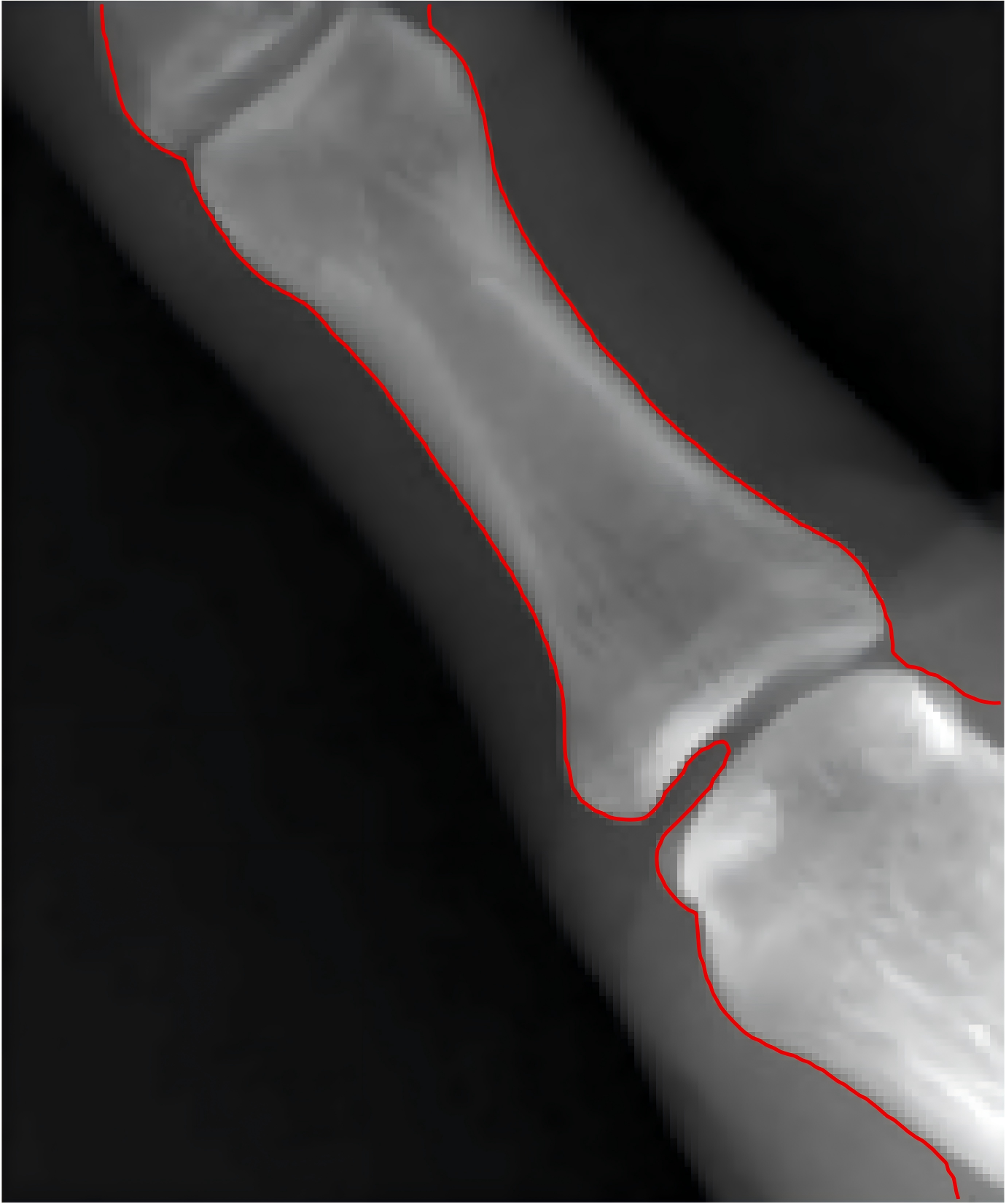}
		\caption{RSF}
		\label{fig:compare-rsf}
	\end{subfigure}
	\hfill
	\begin{subfigure}[t]{0.155\textwidth}
		\centering
		\includegraphics[
		width=\textwidth,
		height=0.13\textheight,
		keepaspectratio
		]{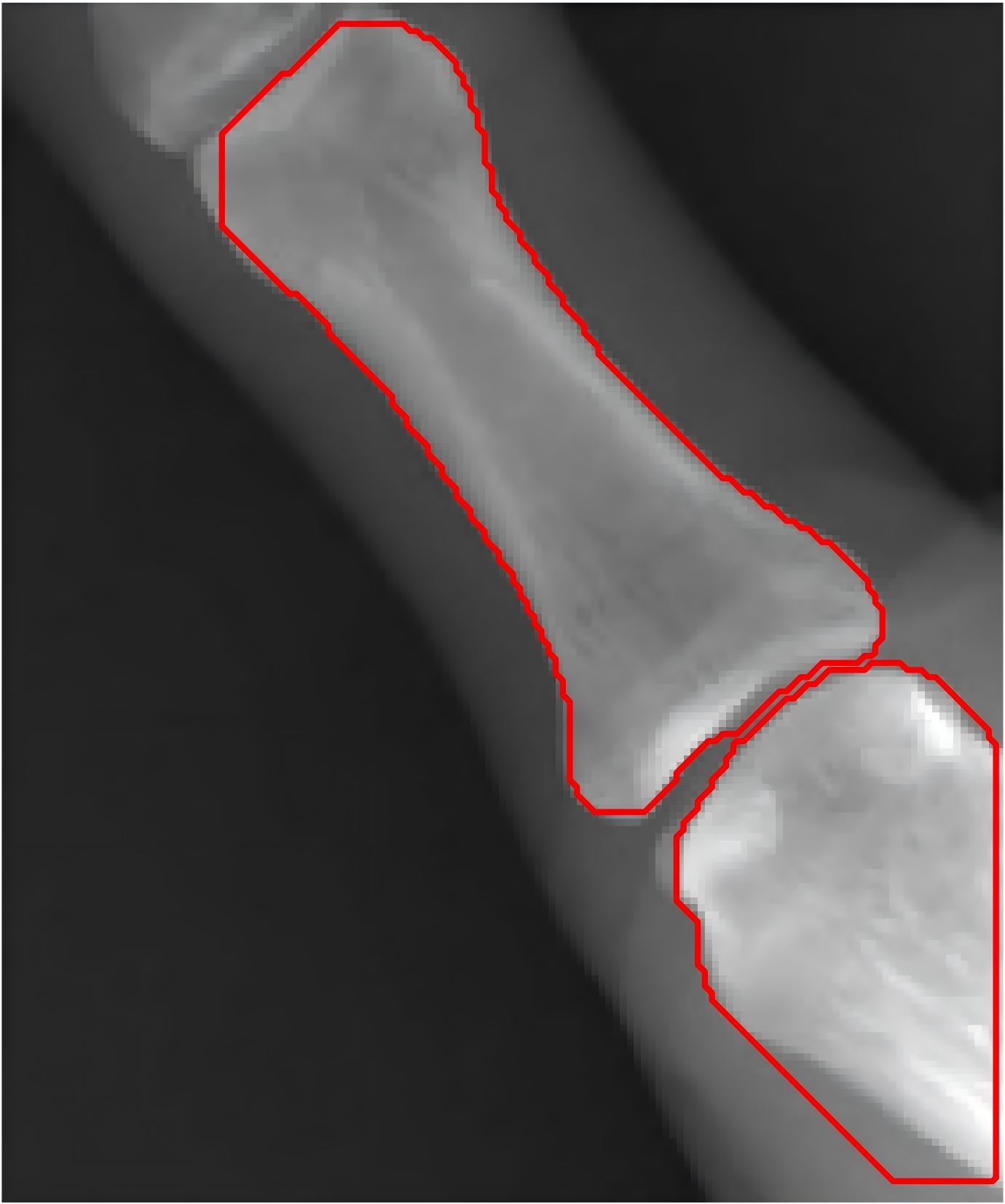}
		\caption{CP-ICTM}
		\label{fig:compare-cpictm}
	\end{subfigure}
	\hfill
	\begin{subfigure}[t]{0.155\textwidth}
		\centering
		\includegraphics[
		width=\textwidth,
		height=0.13\textheight,
		keepaspectratio
		]{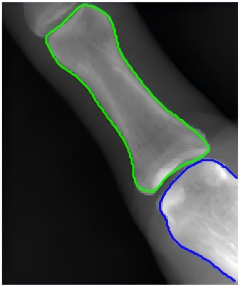}
		\caption{MBE-SSM}
		\label{fig:compare-mbe}
	\end{subfigure}
	
	\caption{Segmentation comparison results of different models.}
	\label{fig:different-models-comparison}
\end{figure}

\begin{figure}[htbp]
	\centering
	
	\begin{subfigure}[t]{0.155\textwidth}
		\centering
		\includegraphics[
		width=\textwidth,
		height=0.26\textheight,
		keepaspectratio
		]{fig/CH_10095_OR.png}
		\caption{Original image}
		\label{fig:compare2-original}
	\end{subfigure}
	\hfill
	\begin{subfigure}[t]{0.155\textwidth}
		\centering
		\includegraphics[
		width=\textwidth,
		height=0.26\textheight,
		keepaspectratio
		]{fig/CH_10095.png}
		\caption{Ours}
		\label{fig:compare2-ours}
	\end{subfigure}
	\hfill
	\begin{subfigure}[t]{0.155\textwidth}
		\centering
		\includegraphics[
		width=\textwidth,
		height=0.26\textheight,
		keepaspectratio
		]{fig/CV_10095.png}
		\caption{CV}
		\label{fig:compare2-cv}
	\end{subfigure}
	\hfill
	\begin{subfigure}[t]{0.155\textwidth}
		\centering
		\includegraphics[
		width=\textwidth,
		height=0.26\textheight,
		keepaspectratio
		]{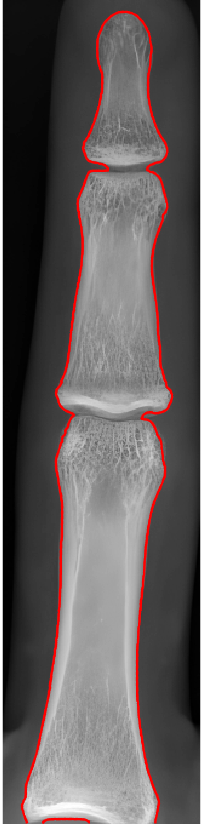}
		\caption{RSF}
		\label{fig:compare2-rsf}
	\end{subfigure}
	\hfill
	\begin{subfigure}[t]{0.155\textwidth}
		\centering
		\includegraphics[
		width=\textwidth,
		height=0.26\textheight,
		keepaspectratio
		]{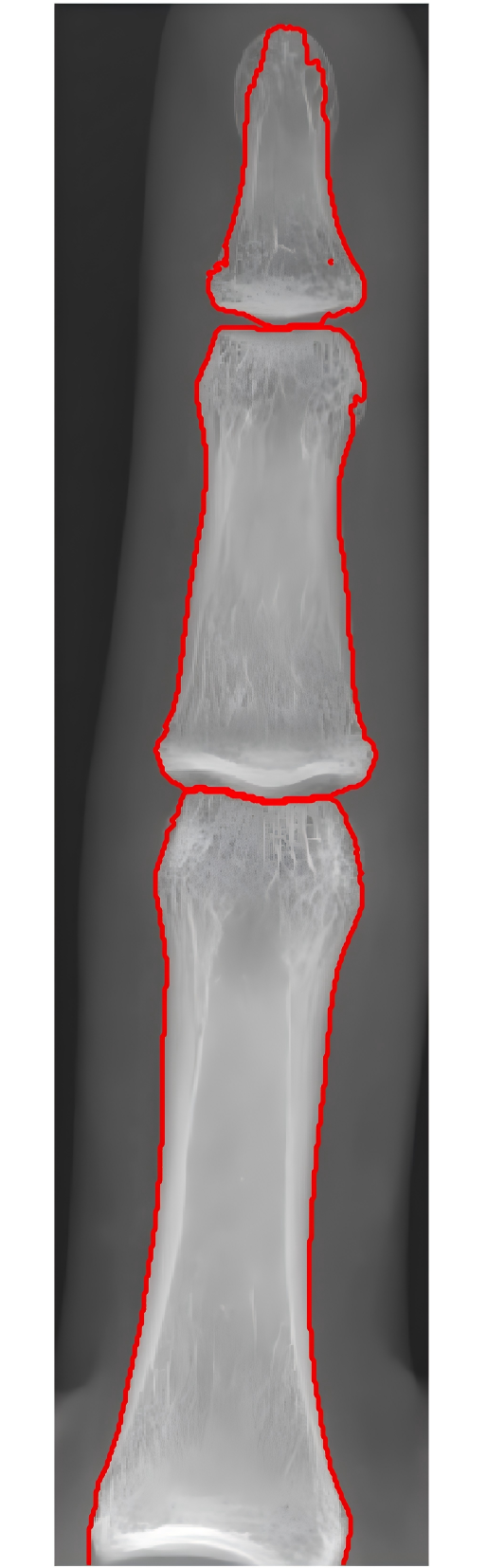}
		\caption{CP-ICTM}
		\label{fig:compare2-cpictm}
	\end{subfigure}
	\hfill
	\begin{subfigure}[t]{0.155\textwidth}
		\centering
		\includegraphics[
		width=\textwidth,
		height=0.26\textheight,
		keepaspectratio
		]{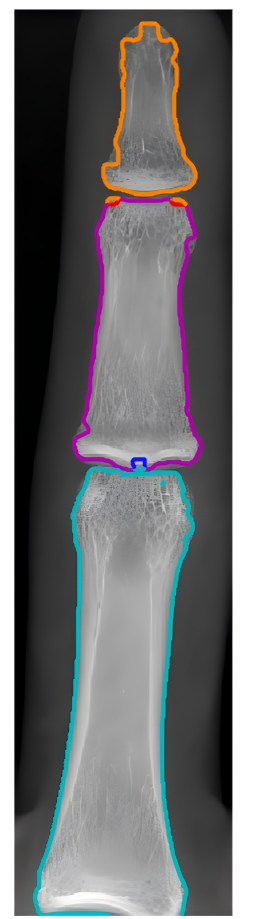}
		\caption{MBE-SSM}
		\label{fig:compare2-mbe}
	\end{subfigure}
	
	\caption{Segmentation comparison results of different models on another image.}
	\label{fig:different-models-comparison-2}
\end{figure}

\subsection{Comparison with deep learning methods on SKI10}
In this subsection, we evaluate the proposed model on the SKI10 dataset~\cite{heimannSegmentationKneeImages2010} and compare it with several representative deep learning-based segmentation methods, including U-Net~\cite{ronnebergerUNetConvolutionalNetworks2015a}, Swin-UNet~\cite{caoSwinUnetUnetlikePure2021}, TransUNet~\cite{chenTransUNetTransformersMake2021}, and MedSegDiff~\cite{wuMedSegDiffMedicalImage}. The segmentation performance is assessed using Dice, IoU, and HD95~\cite{tahaMetricsEvaluating3D2015}. In our experiments, 95 images are extracted from the dataset, with 70 images used for training, 15 for validation, and 10 selected images used for testing.
For a consistent comparison, all deep learning baselines use the same training,
validation, and test partitions. The networks are trained for 150 epochs with a
batch size of 2 and an input size of \(256\times256\). AdamW is used with a fixed
learning rate of \(10^{-4}\) and a weight decay of \(10^{-4}\). The data
augmentation includes random horizontal flipping, random vertical flipping, and
random rotation in the range \([-20^\circ,20^\circ]\). U-Net, Swin-UNet, and
TransUNet are trained with a loss of \(0.5\) cross-entropy plus \(0.5\) Dice loss,
where labels 1 and 2 are treated as separate foreground classes during training
and are merged only when computing the foreground metrics reported below.
For Swin-UNet, the window size is set to 8 to support the \(256\times256\) input
resolution. MedSegDiff retains its diffusion segmentation loss together with the
calibration loss, following its original single-channel foreground formulation.

\begin{table}[htbp]
	\centering
	\caption{Quantitative comparison of different segmentation methods on the SKI10 dataset~\cite{heimannSegmentationKneeImages2010}. Values are reported as mean $\pm$ standard deviation.}
	\label{tab:ski10-comparison}
	\begin{tabular}{lccc}
		\hline
		Method & Dice $\uparrow$ & IoU $\uparrow$ & HD95 $\downarrow$ \\
		\hline
		U-Net 
		& $\mathbf{0.9884}\pm\mathbf{0.0039}$ 
		& $\mathbf{0.9771}\pm\mathbf{0.0077}$ 
		& $\mathbf{2.3495}\pm0.9639$ \\
		
		TransUNet 
		& $0.9874\pm0.0046$ 
		& $0.9751\pm0.0090$ 
		& $2.5469\pm1.6471$ \\
		
		Swin-UNet 
		& $0.9704\pm0.0095$ 
		& $0.9428\pm0.0178$ 
		& $12.4082\pm9.2355$ \\
		
		MedSegDiff 
		& $0.9432\pm0.0219$ 
		& $0.8932\pm0.0378$ 
		& $24.3894\pm7.0927$ \\
		
		Ours 
		& $0.9816\pm0.0045$ 
		& $0.9639\pm0.0087$ 
		& $3.0442\pm\mathbf{0.4827}$ \\
		\hline
	\end{tabular}
\end{table}

The quantitative results in Table~\ref{tab:ski10-comparison} demonstrate that
the proposed method remains highly competitive with representative supervised
deep learning models while avoiding network training and validation. Although
U-Net and TransUNet obtain slightly higher average Dice and IoU values under the
same training protocol, these results rely on data-driven optimization with
labeled training samples. In contrast, the proposed variational method reaches
Dice and IoU values of $0.9816$ and $0.9639$, respectively, without learning
network parameters. More importantly, for the boundary-sensitive HD95 metric, the
proposed method achieves a mean value of $3.0442$ pixels and the smallest
standard deviation of $0.4827$ pixels among all compared methods. This indicates
that the proposed model provides the most stable boundary localization on the
selected SKI10 test images, which is particularly important for medical image
segmentation where local boundary errors can be more clinically relevant than a
small difference in overlap scores.

\begin{figure}[htbp]
	\centering
	\includegraphics[width=0.90\textwidth]{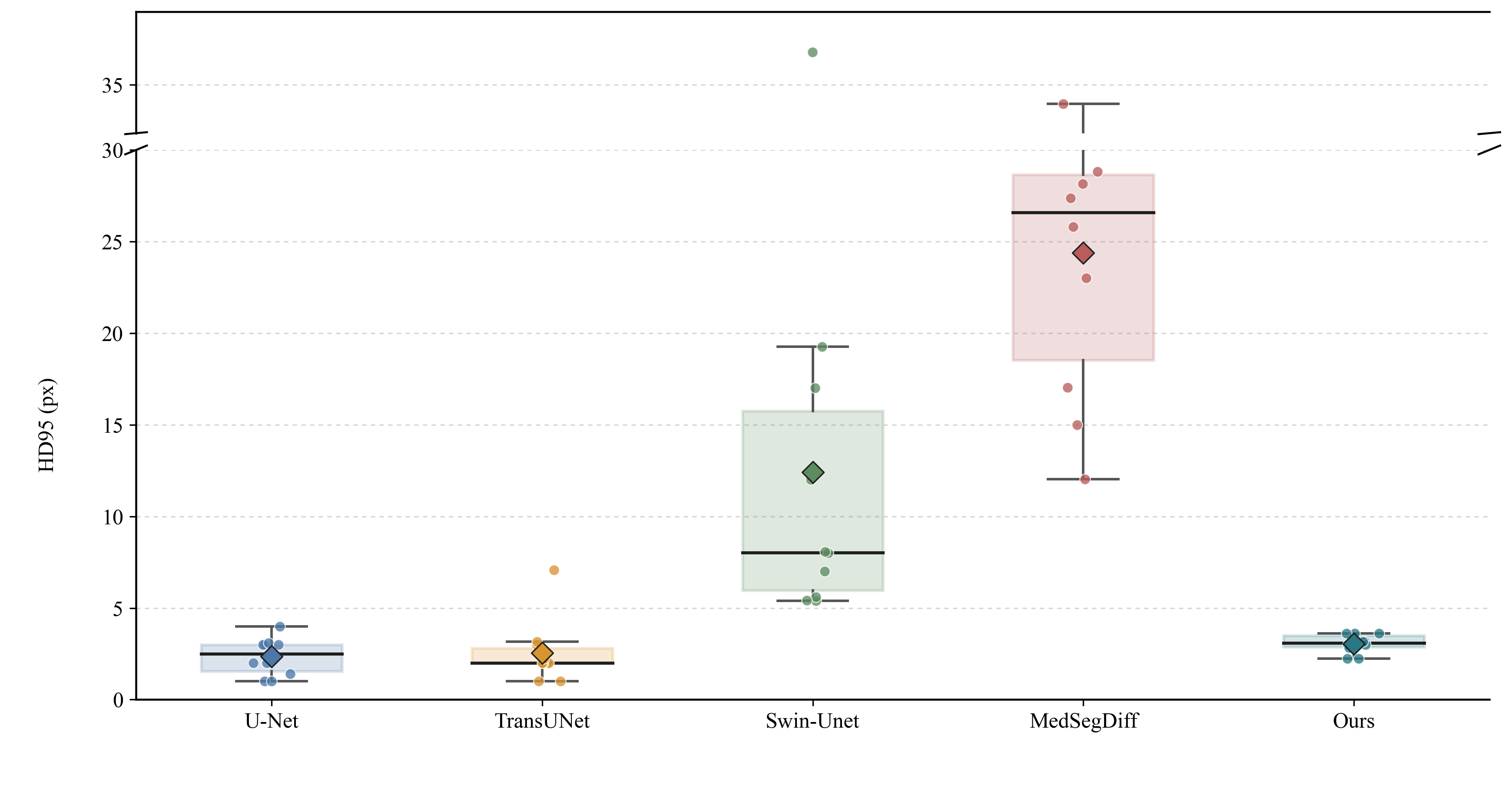}
	\caption{Distribution of HD95 values for different segmentation methods on the
	SKI10 test images. The diamond markers indicate the mean values.}
	\label{fig:ski10-hd95-distribution}
\end{figure}

Figure~\ref{fig:ski10-hd95-distribution} further highlights the boundary
stability of the proposed method. Its HD95 values are tightly concentrated in a
narrow range of roughly \(2\)--\(4\) pixels, producing the most compact
distribution among all methods. This behavior shows that the proposed model is
less sensitive to variations among the selected test images and can maintain a
stable contour accuracy without being re-trained for the dataset. By comparison,
the deep learning methods show larger fluctuations in HD95. Even when U-Net and
TransUNet achieve strong average overlap scores, their wider HD95 distributions
indicate that local boundary errors may still occur in certain cases. Swin-UNet
and MedSegDiff exhibit even broader distributions and larger HD95 values,
suggesting more severe boundary instability. Therefore, Fig.~\ref{fig:ski10-hd95-distribution}
supports the main advantage of the proposed method: it provides consistently
accurate boundary localization rather than relying only on high average overlap
metrics.

The visual comparison in Figure~\ref{fig:knee58-method-comparison} gives a
case-level explanation for these quantitative trends. On the representative test
image, the proposed method produces a smooth and coherent contour and reaches an
HD95 value of $3.0000$ pixels without network training. The extracted bone
regions remain well separated near weak and adjacent anatomical boundaries,
showing the benefit of the Cahn--Hilliard regularization in maintaining stable
interfaces. Although U-Net and TransUNet obtain high Dice and IoU values on this
test image, local boundary leakage or slight contour inconsistency can still
appear near ambiguous bone interfaces. This observation is important because it
shows that high overlap scores from trained networks do not always guarantee
fully stable boundary behavior. Swin-UNet shows more visible boundary
inconsistency, and MedSegDiff exhibits obvious foreground leakage, which is
consistent with their larger HD95 values in
Figure~\ref{fig:ski10-hd95-distribution}.

For the classical variational models, including CV~\cite{chanActiveContoursEdges2001},
RSF~\cite{chunmingliMinimizationRegionScalableFitting2008a}, and
CP-ICTM~\cite{dengConnectedComponentPreservingImage2025},
the visual results in Figure~\ref{fig:knee58-method-comparison} show a more
severe failure on the representative test image. CV and RSF do not form coherent
segmentations of the target bone regions; their contours are attracted by
scattered intensity variations and unrelated anatomical edges, producing
fragmented or misplaced regions rather than meaningful separated bone masks.
CP-ICTM preserves the imposed topological structure to some extent, but this
constraint does not correct the weak-boundary ambiguity in this case. Instead,
the resulting contour follows several irrelevant image structures and fails to
localize the target anatomy. Thus, these classical baselines are not merely
affected by small boundary deviations in Figure~\ref{fig:knee58-method-comparison};
they lose reliable target localization on this test image. In contrast, the
proposed model keeps the segmentation concentrated on the bone regions and
produces coherent interfaces, showing the advantage of the Cahn--Hilliard
regularization for stabilizing contour evolution under weak and ambiguous image
forces. Taken together, Table~\ref{tab:ski10-comparison}
and Figures~\ref{fig:ski10-hd95-distribution}--\ref{fig:knee58-method-comparison}
show that the proposed method combines competitive overlap accuracy, the most
stable HD95 behavior, and clear boundary coherence without requiring a training
stage. These properties make it a strong alternative when robustness, boundary
regularity, and reduced dependence on annotated training data are prioritized.

\begin{figure}[htbp]
	\centering
	\captionsetup[subfigure]{skip=2pt}
	\captionsetup{skip=4pt}
	
	\begin{subfigure}[t]{0.28\textwidth}
		\centering
		{\small Original image and ground truth}\\[1mm]
		\includegraphics[
		width=\textwidth,
		height=0.23\textheight,
		keepaspectratio
		]{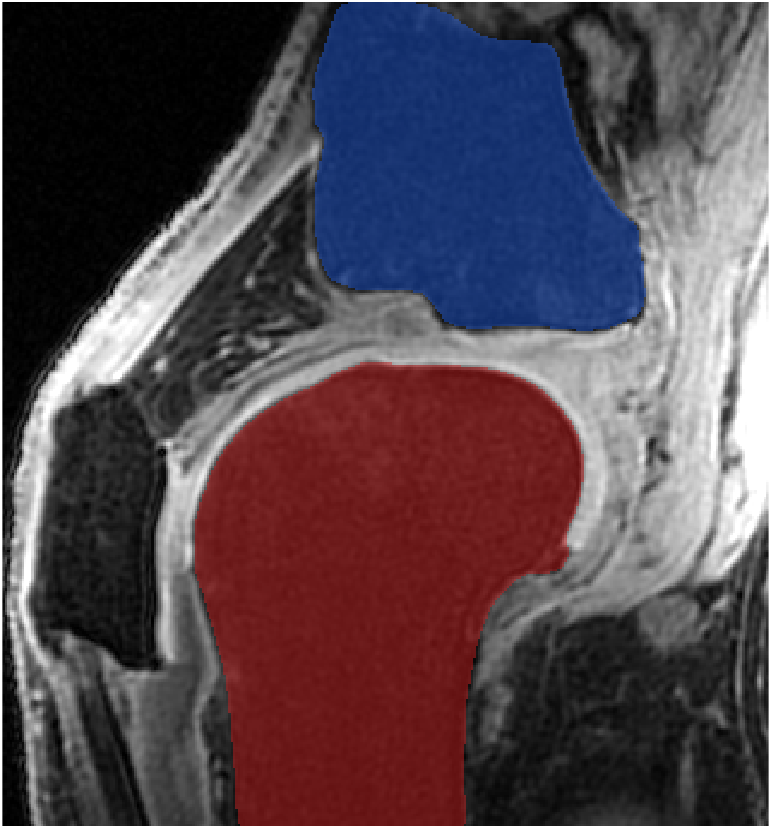}
		\caption{Original image}
		\label{fig:knee58-original}
	\end{subfigure}
	\hfill
	\begin{subfigure}[t]{0.28\textwidth}
		\centering
		{\small Dice = 0.9825, IoU = 0.9657, HD95 = 3.0000}\\[1mm]
		\includegraphics[
		width=\textwidth,
		height=0.23\textheight,
		keepaspectratio
		]{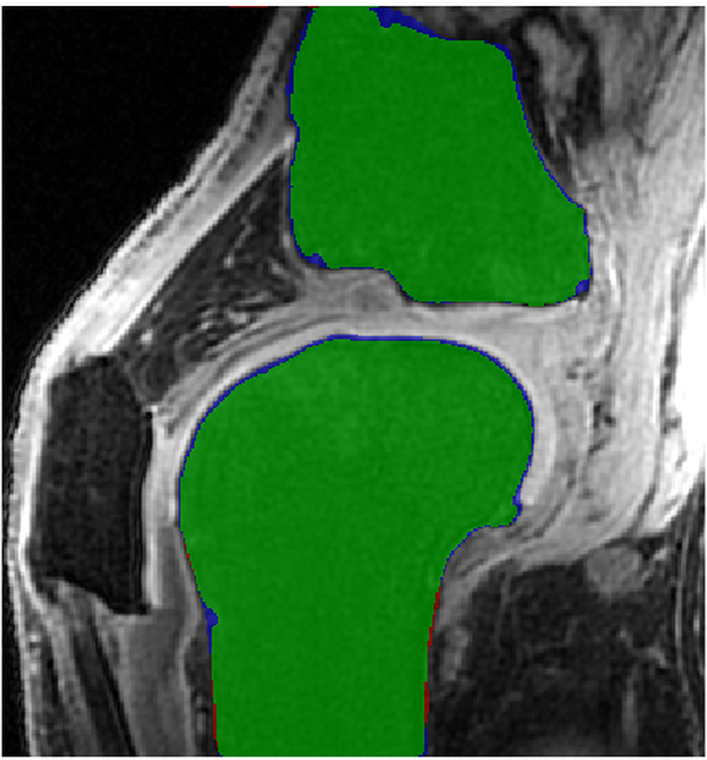}
		\caption{Ours}
		\label{fig:knee58-ours}
	\end{subfigure}
	\hfill
	\begin{subfigure}[t]{0.28\textwidth}
		\centering
		{\small Dice = 0.9871, IoU = 0.9746, HD95 = 3.0000}\\[1mm]
		\includegraphics[
		width=\textwidth,
		height=0.23\textheight,
		keepaspectratio
		]{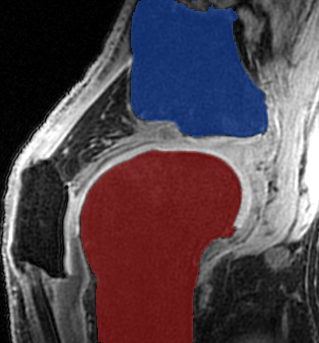}
		\caption{U-Net}
		\label{fig:knee58-unet}
	\end{subfigure}
	
	\begin{subfigure}[t]{0.28\textwidth}
		\centering
		{\small Dice = 0.9865, IoU = 0.9733, HD95 = 3.1623}\\[1mm]
		\includegraphics[
		width=\textwidth,
		height=0.23\textheight,
		keepaspectratio
		]{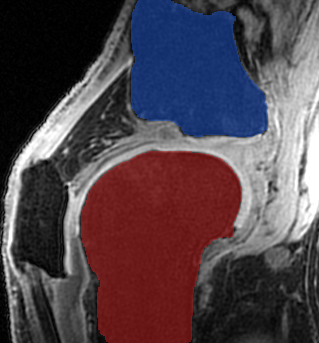}
		\caption{TransUNet}
		\label{fig:knee58-transunet}
	\end{subfigure}
	\hfill
	\begin{subfigure}[t]{0.28\textwidth}
		\centering
		{\small Dice = 0.9724, IoU = 0.9462, HD95 = 7.0000}\\[1mm]
		\includegraphics[
		width=\textwidth,
		height=0.23\textheight,
		keepaspectratio
		]{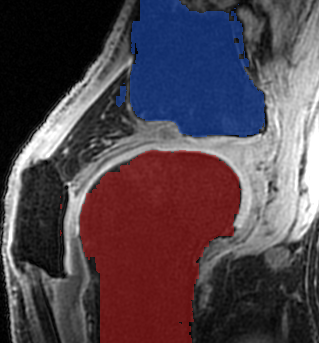}
		\caption{Swin-UNet}
		\label{fig:knee58-swinunet}
	\end{subfigure}
	\hfill
	\begin{subfigure}[t]{0.28\textwidth}
		\centering
		{\small Dice = 0.9521, IoU = 0.9086, HD95 = 23.0000}\\[1mm]
		\includegraphics[
		width=\textwidth,
		height=0.23\textheight,
		keepaspectratio
		]{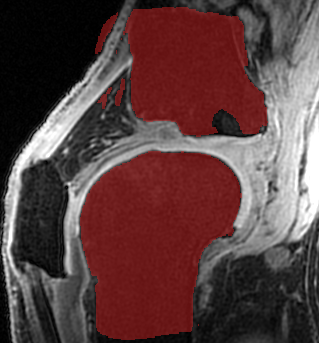}
		\caption{MedSegDiff}
		\label{fig:knee58-medsegdiff}
	\end{subfigure}
	
	\begin{subfigure}[t]{0.28\textwidth}
		\centering
		{\small Dice = --, IoU = --, HD95 = --}\\[1mm]
		\includegraphics[
		width=\textwidth,
		height=0.23\textheight,
		keepaspectratio
		]{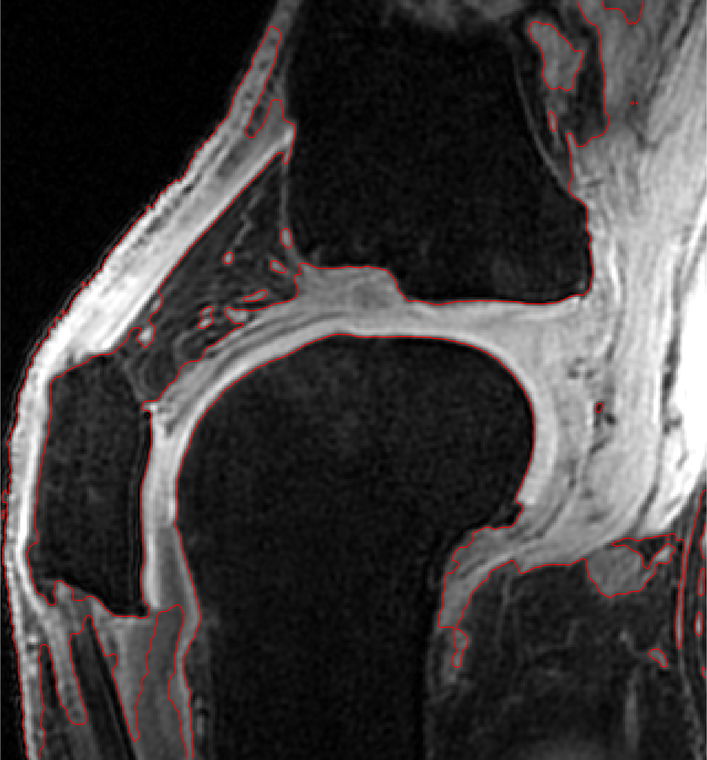}
		\caption{CV}
		\label{fig:knee58-cv}
	\end{subfigure}
	\hfill
	\begin{subfigure}[t]{0.28\textwidth}
		\centering
		{\small Dice = --, IoU = --, HD95 = --}\\[1mm]
		\includegraphics[
		width=\textwidth,
		height=0.23\textheight,
		keepaspectratio
		]{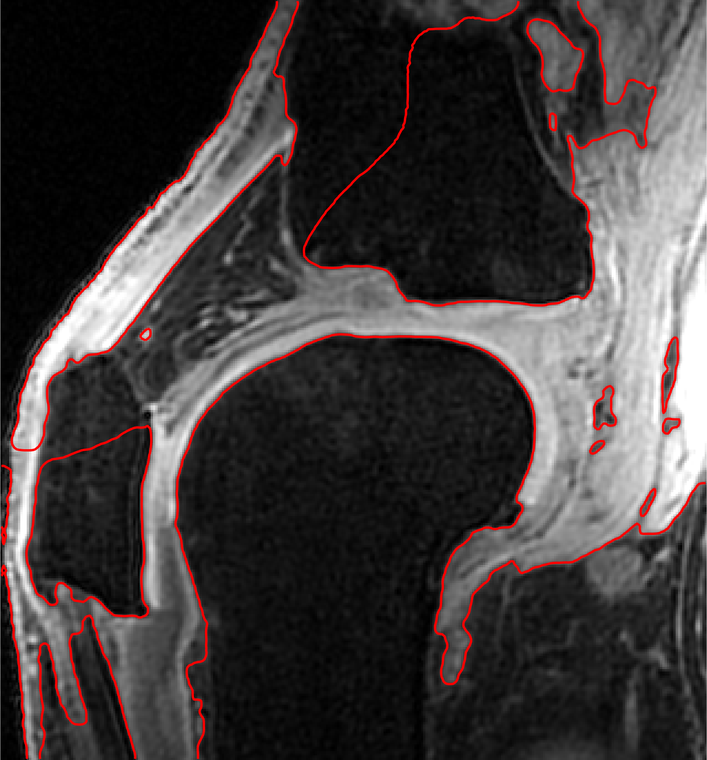}
		\caption{RSF}
		\label{fig:knee58-rsf}
	\end{subfigure}
	\hfill
	\begin{subfigure}[t]{0.28\textwidth}
		\centering
		{\small Dice = --, IoU = --, HD95 = --}\\[1mm]
		\includegraphics[
		width=\textwidth,
		height=0.23\textheight,
		keepaspectratio
		]{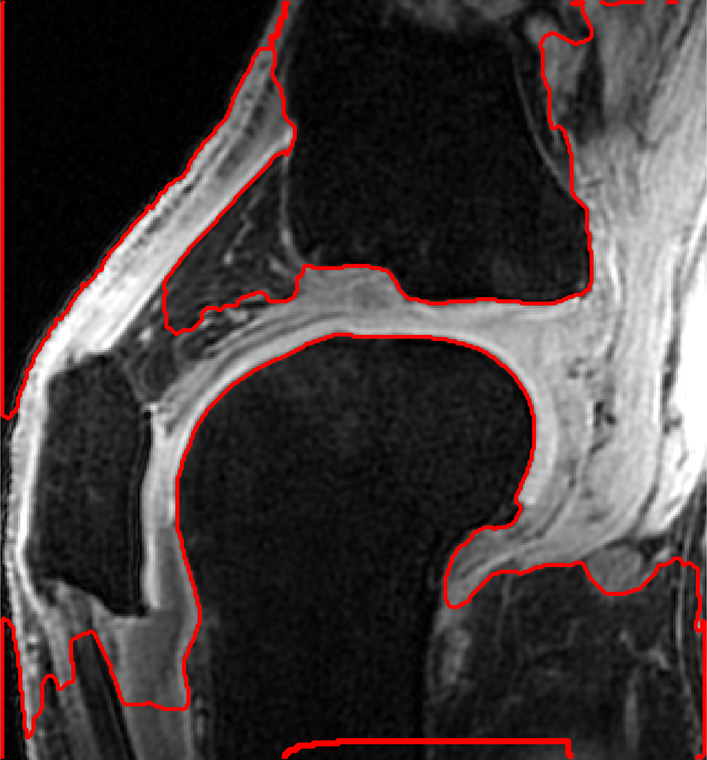}
		\caption{CP-ICTM}
		\label{fig:knee58-cp-ictm}
	\end{subfigure}
	
	\caption{Visual comparison of different segmentation methods on a representative SKI10 image.}
	\label{fig:knee58-method-comparison}
\end{figure}


\section{Conclusions}
\label{sec:conclusion}

In this paper, we developed a smooth Cahn--Hilliard phase-separation model for
weak-boundary segmentation of homogeneous-appearance but semantically distinct
structures. By incorporating phase separation into the softmax-based
segmentation framework, the proposed method was able to continue regulating the
phase variables when the fitting force became weak near a low-contrast
interface. This mechanism reduced undesirable merging between adjacent
homogeneous regions and improved the smoothness and coherence of the resulting
boundaries. The mixed gradient flow avoided the restrictive mass conservation of
the classical Cahn--Hilliard dynamics while retaining its higher-order
regularization effect. The corresponding evolution problem was well posed and
satisfied a continuous energy dissipation law. We also designed a stabilized
SAV scheme coupled with FFT for the proposed model, leading to a linear
implementation and stable computation over a wide range of time steps.
Numerical experiments showed that the Cahn--Hilliard term provided a clear
improvement over the unregularized model in separating nearby structures. The
method was insensitive to moderate variations of the principal parameters and
produced reliable segmentations for noisy, weak-boundary, and complex medical
images. Comparisons with representative variational, topology-preserving,
phase-field, and deep learning methods further demonstrated its competitive
accuracy and consistent boundary localization. Overall, the proposed
regularization provided an effective and mathematically supported approach to
homogeneous-structure segmentation problems in which image contrast alone was
insufficient to distinguish neighboring regions.

\appendix


\section*{Acknowledgments}
This work is partially supported by National Natural Science Foundation of China (U21B2075, 12171123, 12271130) and the Fundamental Research Funds for the Central Universities (2022FRFK060020).

\par

\bibliographystyle{unsrt}
\bibliography{references}

\end{document}